\documentclass[11pt]{article}

\usepackage[margin=1in]{geometry}
\usepackage{wrapfig}
\usepackage{svg}

\title{\LARGE Feature Learning in Linear-Width Two-Layer Networks:\\ Two vs. One Step of Gradient Descent
} 
\author{Behrad Moniri\thanks{The authors are with the University of Pennsylvania, Philadelphia, PA, USA.\\
\phantom{*\quad\;}Correspondence to \texttt{bemoniri@seas.upenn.edu}.} 
\and 
Hamed Hassani}

\usepackage{microtype}
\usepackage{graphicx}
\usepackage{subfigure}
\usepackage{booktabs} 
\usepackage{xspace}

\usepackage[colorlinks=true,linkcolor=black,citecolor=blue,urlcolor=blue]{hyperref}
\usepackage{amsmath,amssymb,amsthm}
\usepackage{mathtools}
\usepackage{thm-restate}

\usepackage{natbib}

\usepackage{amsfonts, bm,tabularx}
\usepackage{grffile, float, latexsym, mathtools}
\usepackage{multicol}
\usepackage{enumitem}
\usepackage{tikz}
\usepackage[parfill]{parskip}

\DeclareMathOperator{\bmtx}{\begin{bmatrix}}
\DeclareMathOperator{\emtx}{\end{bmatrix}}

\DeclareMathOperator{\trace}{Tr}

\DeclareMathOperator{\diag}{diag}

\newcommand{\R}{\mathbb{R}}

\newcommand{\bfA}{\ensuremath{\mathbf{A}}}

\newcommand{\bfE}{\ensuremath{\mathbf{E}}}

\newcommand{\bfG}{\ensuremath{\mathbf{G}}}

\newcommand{\bfI}{\ensuremath{\mathbf{I}}}

\newcommand{\bfO}{\ensuremath{\mathbf{O}}}

\newcommand{\bfQ}{\ensuremath{\mathbf{Q}}}

\newcommand{\bfT}{\ensuremath{\mathbf{T}}}

\newcommand{\bfW}{\ensuremath{\mathbf{W}}}
\newcommand{\bfX}{\ensuremath{\mathbf{X}}}

\newcommand{\bfa}{\ensuremath{\mathbf{a}}}

\newcommand{\bfDelta}{\ensuremath{\boldsymbol{\Delta}}}

\newcommand{\bfu}{\ensuremath{\mathbf{u}}}
\newcommand{\bfv}{\ensuremath{\mathbf{v}}}

\newcommand{\bfx}{\ensuremath{\mathbf{x}}}
\newcommand{\bfy}{\ensuremath{\mathbf{y}}}

\def\st/{\textsuperscript{st}}
\def\nd/{\textsuperscript{nd}}
\def\rd/{\textsuperscript{rd}}
\def\th/{\textsuperscript{th}}

\newcommand{\parant}[1]{\left ( #1 \right )}
\newcommand{\bracket}[1]{\left [ #1 \right ]}

\newcommand{\dx}{\mathrm{d}_{\scriptscriptstyle\mathsf{X}}}

\newcommand{\Ex}{\mathsf{E}}

\newtheoremstyle{mystyle}
  {6pt}
  {6pt}
  {\itshape}
  {}
  {\bfseries}
  {.}
  { }
  {}

\theoremstyle{mystyle}
\newtheorem{theorem}{Theorem}
\newtheorem{proposition}[theorem]{Proposition}
\newtheorem{example}[theorem]{Example}
\newtheorem{corollary}[theorem]{Corollary}
\newtheorem{lemma}[theorem]{Lemma}

\newtheorem*{settings*}{Settings}

\newtheorem{assumption}{Condition}

\def\ep{{\varepsilon}}

\newcommand{\vbeta}{\boldsymbol{\beta}}
\newcommand{\normal}{{\sf N}}

\newcommand\independent{\protect\mathpalette{\protect\independenT}{\perp}}
\def\independenT#1#2{\mathrel{\rlap{$#1#2$}\mkern2mu{#1#2}}}

\date{}

\begin{document}
\maketitle
\begin{abstract}
We study feature learning in two-layer neural networks within the linear-width regime, where the number of hidden neurons, sample size, and input dimension scale proportionally. While recent work has analyzed feature learning via a single step of gradient descent on the first layer weights in this regime, such one-step update schemes are fundamentally limited: the update to the weights is approximately rank-one, captures only a single direction, and requires the target function to have an information exponent of one. In this paper, we go beyond one-step updates to provide a full characterization of the features learned during the \textit{second step} of gradient descent with step-sizes $\eta_1 \asymp N^{\alpha_1}$ and $\eta_2 \asymp N^{\alpha_2}$ for $\alpha_1, \alpha_2 \in [0,0.5)$, where $N$ is the number of hidden neurons. We derive a sharp spectral characterization of the updated weights, demonstrating they behave as a spiked random matrix with multiple outliers, each corresponding to a learned direction. We show that the number of these outliers is determined by the scaling parameters $\alpha_1$ and $\alpha_2$ through  $\lfloor \frac{\alpha_2}{1/2 - \alpha_1} \rfloor$. Furthermore, by analyzing the alignment between these learned directions and the target function, we identify a qualitative gap between training with independent versus reused batches. While independent batches restrict learning to directions with an information exponent of one, batch reuse enables the second update to capture directions even when the information exponent exceeds one, provided that $\alpha_1$ and $\alpha_2$ are chosen properly. This shows that the benefits of batch reuse, previously observed in narrow-width regimes, persist in the high-dimensional linear-width limit as well. By characterizing these early-phase spectral evolutions, our work proposes a tractable framework for studying optimization and feature learning phenomenology in modern overparameterized networks.
\end{abstract}

\section{Introduction}

The ability to learn meaningful \textit{features}—also called \textit{representations}—from raw data is  considered a cornerstone of the success of modern deep learning systems (e.g., \cite{lecun2002gradient, krizhevsky2012imagenet, bengio2013representation, donahue2016adversarial, radhakrishnan2024mechanism}). As a result, understanding feature learning has been a core focus of deep learning theory research. However, the theoretical understanding of feature learning remains incomplete.

Two-layer (fully connected) neural networks are the simplest and one of the most popular models studied in deep learning theory. A two-layer  neural network $f_{\rm NN}: \R^{\dx}\to\R$ with $N \in \mathbb{N}$ hidden neurons is defined as
\begin{align}
    \label{eq:two_layer}
    f_{\rm NN}(\bfx) := \bfa^\top \sigma(\bfW \bfx) \quad \forall \bfx \in \R^{\dx}
\end{align}
where $\bfa \in \R^{N}$ and $\bfW \in \R^{N \times \dx}$ are trainable parameters and $\sigma: \R \to \R$ is an activation function that is applied element-wise.  Consider a supervised learning setting where we are given $n \in \mathbb{N}$ i.i.d.  training samples. Feature learning in this model is the problem of using the training samples to train the first-layer weight matrix $\bfW$. A natural setting to consider is the \textit{high-dimensional proportional limit} where $n$ and $\dx$ tend to infinity, with a proportional rate. Intuitively, this limit reflects the setting where the
dimension of the data and sample size are comparable, which is consistent with many practical scenarios.

The behavior of the two-layer network in \eqref{eq:two_layer} varies qualitatively depending on the width $N$. On one end of the spectrum is the \textit{infinite-width} setting where $N \to \infty$ with $N \gg \dx, n$. In this limit, the model is amenable to theoretical analysis due to its connection to kernel methods. See e.g., the Neural Tangent Kernel (NTK) setting \cite{jacot2018neural}. However, because of the same connections to kernel methods, in this settings, features are often ``stuck" at initialization and feature learning is largely absent \cite{ghorbani2021linearized,ghorbani2021neural}. On the other end of the spectrum is the \textit{narrow-width} setting where $N \ll \dx$, even though $n, \dx$ are large. Feature learning in this setting has been studied extensively in recent years (see e.g., \cite{cui2025high} for a review). However, the narrow-width setting is often not realistic, as modern deep learning systems are often highly \textit{overparameterized}. In fact, many  considerations, especially regarding the proper scaling of parameters, that are also highly relevant in deep learning practice, only emerge when the width is assumed to be large \cite{yang2020feature_learn}. A popular intermediate setting is the \textit{linear-width} setting with $N \to \infty$ with the same rate as $n, \dx$. This regime allows for the study of overparameterized models while ensuring that the model does not collapse into fixed-feature kernel methods. In this paper, we focus on this limit.

In the linear-width setting, the problem has been extensively studied in the case when the first-layer weights $\bfW$ is initialized randomly at $\bfW_0 \in \R^{N \times \dx}$ with
$[\bfW_0]_{ij} \sim \normal(0, 1/\dx)$, and is kept fixed throughout training and only the second layer $\bfa$ is trained; i.e., the \textit{random features model} \citep{RahimiRecht} (see e.g., \cite{mei2022generalization,adlam2019random,mel2021anisotropic}, etc.). Random features models have emerged as a highly successful framework for studying various phenomena in deep learning, offering rigorous insights into the role of overparameterization and regularization (see e.g., \cite{adlam2020understanding, lin2021causes, tripuraneni2021covariate, hassani2022curse, disagreement,clarte2023double}). However, the first-layer weights are assumed to be randomly generated, and then fixed; thus this model is still fundamentally limited and no feature learning happens. A consequence of this limitation is that although in principle the model can represent nonlinear functions, it only learns the {linear} component of the function generating the samples; a property often referred to as {Gaussian Equivalence} (see e.g., \cite{goldt2022gaussian, mei2022generalization,hu2022universality, montanari2022universality, hassani2022curse}, etc.).

Given these limitations, several recent studies have studied two-layer networks in the linear-width regime, beyond random features. The pioneering work \cite{ba2022high} considered a two-layer network trained using a layer-wise training procedure. First, fixing $\bfa$ at initialization $\bfa_0 \in \R^N$, the first-layer weight $\bfW$ is updated with \textit{one step of gradient descent}, given by
\begin{align*}
    \bfW_1 = \bfW_0 - \eta \frac{\partial \mathcal{L}}{\partial \bfW} \Big|_{\bfW_0, \bfa_0}
\end{align*}
where $\mathcal{L}$ is the mean squared error and $\eta$ is the step-size. Then $\bfW$ is fixed at $\bfW_1$ and $\bfa$ is trained using ridge regression. \cite{ba2022high} showed that the spectrum of $\bfW_1$ contains of an outlying \textit{spike}\footnote{Using terminology from random matrix theory \cite{bai2010spectral,yao2015large}.}, which results in an alignment between the first-layer weights and the linear component of the teacher model. This results in an improvement  over random features models when the target function is a single-index model--a function that only depends on its input through projection on a single direction--with a non-zero linear component. Consequently,  \cite{moniri_atheory2023} studied the role of the step-size $\eta$ assuming $\eta = \Theta(N^\alpha)$ for $\alpha \in [0, 1/2)$. They showed that by varying $\alpha$, phase transitions occur in the training and test errors achieved by the model at $\alpha = \frac{\ell - 1}{2\ell}$ for $\ell \in \mathbb{N}$, where Gaussian Equivalence progressively breaks. 

\begin{figure}
    \centering
    \includegraphics[width=\linewidth]{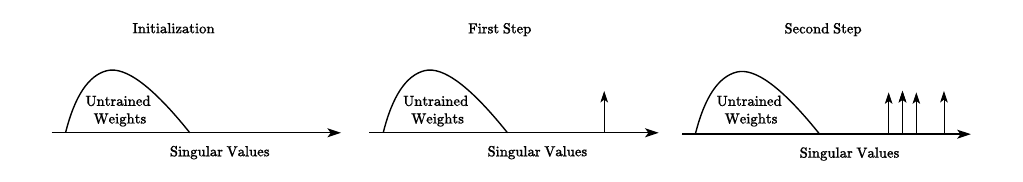}
    \caption{Histogram of the singular values of the weight matrix $\bfW$ (left) at initialization, (middle) after the first step, and (right) after the second step. After the first update, there is at most one extra outlier (see Proposition~\ref{prop:one-step} for details). After the second step, depending on the step-size, more outliers can emerge corresponding to new learned directions (see Theorem~\ref{thm:W2_expansion} for details).}
    \label{fig:cartoon}
\end{figure}

While the one-step updated model has proven instrumental in exploring feature-learning-dependent aspects of modern machine learning, such as optimization design and studying emerging empirical phenomena \cite{zhang2025concurrence, moniri2025mechanisms}, it still remains fundamentally constrained. Specifically, the update is shown to be approximately rank-one, limiting the model to learning a single direction of the target function. Moreover, for the model to have improvements over random features, the target function should have a non-zero linear component. In this paper, we extend this analysis to subsequent gradient descent update on the first-layer weights to uncover the fundamental structural changes it induces. In particular, we ask the following questions:
\begin{quote}
    \textit{What features does the second step of gradient descent learn? How does the spectrum of the first-layer weight matrix evolve as a results of the second update? How do the learned features and the spectrum of the updated weights depend on the step-sizes? Does the second update enable learning beyond a single direction, and functions without a linear component?
    }
\end{quote}

\subsection{Contributions}
Toward answering the above questions, in this paper, we make the following contributions:

\begin{itemize}
    \item We study feature learning in two-layer linear width neural networks. Specifically, we extend the one-step training procedure introduced in \cite{damian2022neural,ba2022high} and update $\bfW$ using two steps of update with step-sizes $\eta_1$ and $\eta_2$. We consider a regime with $\eta_1 \asymp N^{\alpha_1}$ and $\eta_2 \asymp N^{\alpha_2}$, $\alpha_1, \alpha_2 \in [0,1/2)$  and examine how the learned features change with $\alpha_1, \alpha_2$.
    
    \item We present a spectral analysis of the updated first-layer weight matrix. We show that after the second update, this matrix can be approximated in operator norm by the initialized weight $\bfW_0$ plus a low-rank perturbation matrix, controlled by the scaling parameters $\alpha_1, \alpha_2$. The right singular vectors of the low-rank term lie in the span of a collection of $\Lambda = \lfloor\frac{\alpha_2}{1/2 - \alpha_1} \rfloor$ vectors $\hat\vbeta_{2;k},$ for $k \in \{1, \dots, \Lambda\}$. Figure~\ref{fig:cartoon} illustrates this finding. See Theorem~\ref{thm:W2_expansion} for details. 
    
    \item  Next, we study how the learned directions $\hat\vbeta_{2;k}$ align with the relevant directions in the ground truth target function. We show that the answer  changes significantly depending on whether the same or two independent sets of training data has been used to compute the gradients in the first and second iteration.
        
    \item We show that when both steps use the same set of data to compute the gradient, the model can learn relevant directions even if the target function has no linear component. Whereas if a fresh independent set of samples is used for the second update, the model is still restricted to learning functions that have a linear part. 
    \item Finally, we provide numerical simulations and show that they empirically verify our theory.
\end{itemize}

\section{Prior Work}
\label{sec:prior_work}
In this section, we review the relevant prior work in this area.
\paragraph{Random Features Model.} The training and test errors of random features models were studied in \cite{mei2022generalization,hu2022universality,adlam2020neural,adlam2019random}. \cite{BoschPH23} extends the analysis to deep random features models. These results were used to study various deep learning phenomena such as double descent \cite{mei2022generalization, adlam2020understanding,lin2021causes}, robustness to adversarial attacks \cite{hassani2022curse}, and performance under distribution shift \cite{tripuraneni2021covariate,disagreement}. This line of work
builds upon nonlinear random matrix theory  \cite{pennington2017nonlinear, fan2019spectral, louart2018random, benigni2021eigenvalue} that study the spectrum of the feature matrix of two-layer neural networks at initialization.

\paragraph{One-Step Updated Models.} \cite{ba2022high, wang2022spectral} study two-layer networks in the linear-width limit when the first-layer is updated with one step of gradient descent. \cite{moniri_atheory2023,cuiasymptotics,dandi2024random} extend these results to different regimes of step-size and show how the gaussian equivalence property  of random features models \cite{goldt2022gaussian,loureiro2021learning,montanari2022universality} breaks under the one step update. \cite{dandi2023learning} show that with $n = \Theta(\dx)$, it is only possible to learn a single direction of a teacher function. Also, learning a single-index function with information exponent (the index of the first non-zero coefficient in the Hermite expansion) $\kappa$, requires $n = \Theta(\dx^\kappa)$. \cite{nichani2023provable,wanglearning} study three-layer neural networks when first layer is random and the second layer is updated via one step of gradient descent. This model was consequently used in \cite{zhang2025concurrence} for  optimizer design, and to study weak-to-strong generalization \cite{moniri2025mechanisms}. The random matrix results developed in these papers were later extended in \cite{wang2024nonlinear, moniri2024signal, guionnet2023spectral,feldman2025spectral}. 

\paragraph{Narrow Neural Networks.} \cite{dandibenefits,lee2024neural, arnaboldi2024repetita} show that a two-layer network trained with variants of gradient descent with reused training batches used to compute each update, can learn certain single-index models without a linear component to small population error with $n = \tilde\Theta(\dx)$, improving over one-pass SGD \cite{arous2021online}.

\paragraph{Other Deep Learning Theory Setups.}
Random features models in other asymptotic regimes have been used to study various other properties of deep learning models (see e.g., \cite{bombari2023stability,bombari2025privacy,bombari23robustness,iurada2025law,medvedev2025weak}, etc.). Two-layer neural networks have been studied extensively in the mean-field regime (see e.g., \cite{chizat2018global, mei2018mean,pmlr-v99-mei19a,sirignano2020mean,rotskoff2022trainability}, etc.), and the neural tangent kernel (NTK) regime (see e.g., \cite{jacot2018neural,lee2019wide,huang2020dynamics}, etc.).

\subsection{Notation}
We denote vector quantities by bold lower-case, and matrix quantities by bold upper-case. We use $\|\cdot\|_{\rm op}$
$\|\cdot\|_{\rm Fr}$ to denote the operator and Frobenius norms for matrices and $\|\cdot\|_{p}$ to be the $\ell_p$ norm for vectors. For $d\in \mathbb{N}$, we define $[d]:=\{1, \ldots, d\}$. For a matrix $\mathbf{A}$ and $k \in \mathbb{N}$, we define $\mathbf{A}^{\odot k}:=\mathbf{A} \odot \ldots \odot \mathbf{A}$ as the matrix of the $k$-th powers of the elements of $\mathbf{A}$.  We use $O(\cdot)$ and $o(\cdot)$ for the standard big-O and little-o notation.  For positive sequences $\left(A_n\right)_{n \geq 1},\left(B_n\right)_{n \geq 1}$, we write $A_n=\Theta\left(B_n\right)$ (or $A_n \asymp B_n$) if there is $C, C^{\prime}>0$ such that $C B_n \geq A_n \geq C^{\prime} B_n$ for all $n$. We use $O_{\mathbb{P}}(\cdot), o_{\mathbb{P}}(\cdot)$, and $\Theta_{\mathbb{P}}(\cdot)$ for the same notions holding in probability. We use $\to_{\mathbb{P}}$ for convergence in probability. We use $\Ex[\cdot]$ to denote expectations and use $\normal$ to denote the Gaussian distribution.

\section{Problem Setting}
\label{sec:problem_setting}

In this paper, we consider a supervised learning setup where we are given with i.i.d. training samples $\left(\bfx_i, y_i\right) \in \R^{\dx} \times \mathbb{R}$, for $i \in [n]$, where $\dx$ is the feature dimension and $n \in \mathbb{N}$ is the sample size.  Assume that the data is generated according to the sum of $M \in \mathbb{N}$ single-index models

\begin{align}
    \label{eq:data_gen}
    \bfx_i \stackrel{\text {i.i.d.}}{\sim} \normal\left(0, \bfI_{\dx}\right) \text {, and } y_i= \sum_{k = 1}^{M}g_k\left(\boldsymbol{x}_i^\top \vbeta_{\star,k}\right)+\ep_i,
\end{align}
in which  $g_k:\R\to\R$ for $k \in [M]$ are an unknown link functions, $\vbeta_{\star,k} \in \R^{\dx}$ for $k \in [M]$ are a set of $M$ orthonormal vectors (in $\ell_2$),  and $\ep_i \stackrel{\text {i.i.d.}}{\sim} \normal\left(0, \sigma_{\varepsilon}^2\right)$ is an independent additive noise.  We stack these variables in matrix form as $\bfX \in \R^{n \times \dx}$, $\bfy \in \R^n$, and $\boldsymbol{\ep}\in \R^n$.

In order to predict outcomes for unlabeled examples at test time, we fit a two-layer fully-connected  neural network $f_{\rm NN}: \R^{\dx}\to\R$ with $N \in \mathbb{N}$ hidden neurons, which is defined as
\begin{align*}
    f_{\rm NN}(\bfx) := \bfa^\top \sigma(\bfW \bfx) \quad \forall \bfx \in \R^{\dx}
\end{align*}
where $\bfa \in \R^{N}$ and $\bfW \in \R^{N \times \dx}$ are trainable parameters and $\sigma: \R \to \R$ is a pre-determined activation function, that is applied element-wise to $\bfW\bfx \in \R^N$.

To train the two-layer neural network \eqref{eq:two_layer} using the training samples, we first initialize the model weights $\bfW$
as $\bfW_0 \in \R^{N \times \dx}$ with 
$[\bfW_0]_{ij} \sim \normal(0, 1/\dx)$, and  $\bfa$ as  $\bfa_0 \in \R^N$ with $[\bfa_0]_{i} \sim \normal(0, 1/N)$. 
Following the prior work in this area (e.g., \cite{damian2022neural,ba2022high,moniri_atheory2023,dandi2023learning}, etc.), we consider a layer-wise training procedure where we first fix $\bfa$ at $\bfa_0$ and train the first-layer weight matrix $\bfW$ using an iterative algorithm. Note that such layer-wise training is common in the modern deep learning optimization practice (see e.g. \cite{zhang2025concurrence} for more discussions).

We consider the \textit{correlation loss} for the updates (see e.g., \cite{arnaboldi2024repetita,lee2024neural, oko2024learning}, etc.)
 which for training samples $\bar\bfX \in \R^{\bar{n} \times \dx}$ and $\bar\bfy \in \R^{\bar{n}}$ is given by
\begin{align}
    \label{eq:corr_loss}
    \widehat{\mathcal{L}}_{\rm Corr}(\bfW, \bfa; \bar{\bfX}, \bar{\bfy}) := 1 - \frac{1}{\bar{n}} \left|\bar{\bfy}^\top \sigma(\bar{\bfX}\bfW^\top)\bfa\right|.
\end{align}
Note that the squared loss with the specific initialization used in \cite{ba2022high,moniri_atheory2023, cuiasymptotics, dandi2024random, zhang2025concurrence} is effectively equivalent to update with the correlation loss. The gradient of the loss function with respect to $\mathbf{W}$ can be written as
\begin{align}
    \label{eq:corr_loss_gradient}
    \nabla_{\mathbf{W}} \widehat{\mathcal{L}}_{\rm Corr}(\bfW, \bfa; \bar\bfX, \bar\bfy) = -\frac{1}{\bar{n}} \text{sign}\left(\bar\bfy^\top \sigma(\bar{\bfX}\mathbf{W}^\top)\mathbf{a}\right) \left[ (\bfa\bar\bfy^\top)  \odot \sigma'(\mathbf{W}\bar\bfX^\top) \right] \bar\bfX \in \R^{N\times \dx}
\end{align}
where $\sigma':\R\to\R$ is the derivative of the activation function which is applied element-wise.

In this paper, we update the first-layer weight matrix $\bfW$ according to
\begin{equation}
    \label{eq:iteration}
    \begin{split}
        \mathbf{W}_{1} &= \mathbf{W}_{0} - \eta_{1} \nabla_{\mathbf{W}} \widehat{\mathcal{L}}_{\rm Corr}(\mathbf{W}_{0}, \mathbf{a}_0; \mathbf{X}_{1}, \mathbf{y}_{1}) \\
        \mathbf{W}_{2} &= \mathbf{W}_{1} - \eta_2 \nabla_{\mathbf{W}} \widehat{\mathcal{L}}_{\rm Corr}(\mathbf{W}_{1}, \mathbf{a}_0; \mathbf{X}_{2}, \mathbf{y}_{2})
    \end{split}
\end{equation}
where $(\mathbf{X}_t, \mathbf{y}_t) \in \mathbb{R}^{\xi_t n \times \dx} \times \mathbb{R}^{\xi_t n}$ represents the data batch of $\xi_t n$ samples and $\eta_t$ denote the step-size used to compute the update at iteration $t \in [2]$. Note that unlike prior work \cite{tsiolisinformation,lee2024neural} that require a normalization of the weights after each step or interpolation steps, here we study vanilla gradient descent without any modifications.

 Work in the narrow-width regime show that reusing the same batch of samples, or using fresh samples in different iterations of the update can result in qualitatively different behaviors (see e.g., \cite{dandibenefits, lee2024neural,tsiolisinformation}, etc.). To investigate this in our linear-width regime, here we also consider two distinct settings in this paper. 

\begin{settings*} We consider the following two settings in the paper.
    \begin{itemize}
    \item \textbf{Reused Batch Setting:} We use all of the  samples $(\bfX, \bfy) \in \R^{n \times \dx} \times \R^{n}$ to compute the gradient at each iteration. In other words $\bfX_t = \bfX$ and $\bfy_t = \bfy$ and $\xi_t = 1$ for all $t \in [2]$.

    \item \textbf{Fresh Batch Setting:} We partition the  samples $(\bfX, \bfy) \in \R^{n \times \dx} \times \R^{n}$ into two \textbf{disjoint} batches $(\mathbf{X}_1, \mathbf{y}_1) \in \R^{\xi_1 n \times \dx} \times \R^{\xi_1 n}$ and $(\mathbf{X}_2, \mathbf{y}_2) \in \R^{\xi_2 n \times \dx} \times \R^{\xi_2 n}$ with  $\xi_1 + \xi_2  = 1$.
\end{itemize}
\end{settings*}

\subsection{Conditions}
\label{sec:conditions}
In this paper, we study the problem in a \textit{high-dimensional proportional limit} with a \textit{linear-width} neural network. Moreover, we consider the step-size scaling proposed in \cite{moniri_atheory2023}. The limit is formally defined as follows.

\begin{assumption}
    \label{assumption:high_dim} Throughout the paper we consider the following asymptotic setting for the problem:
    \begin{itemize}
        \item We have  $n, \dx, N \to \infty$ with $\dx \asymp n \asymp N$. We let $\phi := \lim \dx/n$.
      
        \item  We have $\eta_1\asymp N^{\alpha_1}$ and  $\eta_2 \asymp N^{\alpha_2}$ where  $\alpha_1, \alpha_2 \in [0, 1/2)$ are constants.
  
    \end{itemize}
\end{assumption}
Prior work such as \cite{dandibenefits, lee2024neural, tsiolisinformation} that study feature learning in two-layer neural networks after multiple steps of update similarly assume $\dx \asymp n$. However, they consider a network with  $N \ll \dx$. The steps-size scaling required for a wide model can be significantly different than the required step-sizes in a narrow network \cite{yang2021tensor}.  In this paper, we consider the more realistic \textit{linear-width} setting with $N\asymp \dx\asymp n$ following \cite{wang2022spectral, ba2022high,hu2022universality, moniri_atheory2023, cuiasymptotics} and study proper scaling for the step-sizes for feature learning in this asymptotic regime. This is limit better reflects modern highly overparameterized machine learning systems.

Let $\{H_i\}_{i \geq 1}$ be the probabilist's Hermite polynomials on $\mathbb{R}$ defined by
\begin{align*}
    H_i(x)=(-1)^i \exp \left(x^2 / 2\right) \frac{d^i}{d x^i} \exp \left(-x^2 / 2\right), \quad \text{ for } x \in \R. 
\end{align*}
The first few Hermite polynomials are $H_1(x)=x$, $H_2(x)=x^2-1$, and $H_3(x) = x^3 - 3x$.

\begin{assumption}
    We assume that the activation function $\sigma: \mathbb{R} \rightarrow \mathbb{R}$ and the teacher functions $g_k: \mathbb{R} \rightarrow \mathbb{R}$ are polynomials with Hermite coefficients $\{c_{i}\}_{i = 1}^{L}$ and $\{c_{g_k,i}\}_{i = 1}^{L_\star}$ respectively; i.e.,
    \begin{align*}
        &\sigma(z)=\sum_{i=1}^{L} c_i H_i(z), \quad c_i=\frac{1}{i!} \Ex_{Z \sim \normal(0,1)}\left[\sigma(Z) H_i(Z)\right]\\
        &g_k(z)=\sum_{i=1}^{L_\star} c_{g_k, i} H_i(z), \quad c_{g_k, i}=\frac{1}{i!} \Ex_{Z \sim \normal(0,1)}\left[\sigma_{\star}(Z) H_i(Z)\right]
    \end{align*}
Moreover, we denote the information-exponent (i.e., the index of the first non-zero Hermite coefficient) of a function $g :\R \to \R$ by $\tau(g) \in \mathbb{N}$.
\end{assumption}
We make the above polynomial assumptions for simplicity of expressions. However, we note that these assumption can be lifted and similar results also hold for the case where the Hermite coefficients $c_i$ and $c_{g_k, i}$ decay fast enough (see e.g. conditions 2.3 and 2.4 in \cite{moniri_atheory2023}).

In the next section, we fully characterize the spectrum of the updated weight matrices $\mathbf{W}_1$ and $\mathbf{W}_2$ and study what \textit{features} from the target function are encoded in their spectrum. We also study how  the choice of step-size affects these learned features.

\section{Spectrum of the Weight Matrix}
To start, we first summarize the results from the prior work \cite{ba2022high} that characterizes the weight matrix after the first iteration of update.
\begin{proposition}
    \label{prop:one-step}
    Under the conditions of Section~\ref{sec:conditions}, the gradient matrix evaluated for the batch $(\bfX_1, \bfy_1) \in \R^{\xi_1{n}\times \dx}\times \R^{\xi_1{n}}$ at initialization $(\bfW_0, \bfa_0)$ satisfies
    \begin{align*}
         \Big\|\nabla_{\mathbf{W}} \widehat{\mathcal{L}}_{\rm Corr}(\mathbf{W}_{0}, \mathbf{a}_0; \bfX_1, \bfy_1) + c_1 \bfa_0 \hat\vbeta_1^\top\Big\|_{\rm op} = \tilde{O}_{\mathbb{P}}\parant{\frac{1}{\sqrt{N}}},
    \end{align*}
        in which  $\hat\vbeta_1 := \frac{1}{\xi_1{n}}\bfX_1^\top \bfy_1$. Moreover, the updated first-layer weight matrix after the first step of \eqref{eq:iteration} is given by $\bfW_1 = \bfW_0 + c_1 \eta_1 \bfa_0 \hat\vbeta_1^\top + \bfDelta_1$ with $\|\bfDelta_1\|_{\rm op} = o_{\mathbb{P}}(1)$.
\end{proposition}
Note that $\|\bfW_1\|_{\rm op} = O_{\mathbb{P}}(1)$ and $\|c_1\eta_1\bfa_0 \hat\vbeta_1^\top\| = O_{\mathbb{P}}(\eta_1) = O_{\mathbb{P}}(N^{\alpha_1})$, thus $\bfDelta_1$ is negligible. 
This result fully characterizes the updated weight $\bfW_1$. It shows that at initialization, the gradient is approximately rank-one with the right singular-vector approximately equal to  $\hat\vbeta_1$, and the updated weight matrix $\bfW_1$ follows a \textit{spiked random matrix model}. 
Based on the proposition above, the vector $\hat\vbeta_1$ is the direction learned by the first step of gradient descent.

A one-step update is fundamentally limited in two ways. First, the update is approximately rank-one; thus, it can only learn a single direction $\hat\vbeta_1$, even though the outputs depend on projections across $M$ distinct and orthogonal directions $\{\vbeta_{\star, k}\}_{k \in [M]}$. Second, note that based on the conditions in Section~\ref{sec:conditions}, the target function can written as
\begin{align*}
    \sum_{k = 1}^{M} g_k(\bfx^\top\vbeta_\star^\top) =  \bfx^\top\parant{\sum_{k = 1}^{M}c_{g_k, 1}\vbeta_{\star,k}} + \sum_{k = 1}^{M} g_{k, \perp}( \bfx^\top \vbeta_{\star,k})
\end{align*}
with $\Ex_{\bfx \sim \normal(0, \bfI_{\dx})}\left[(\bfx^\top\bfu)\, g_{k, \perp}(\bfx^\top\bfu)\right] = 0$, $\|\bfu\|_2 = 1$. The vector $\hat\vbeta_1 = ({\xi_1 n})^{-1} \bfX_1^\top \bfy_1$ is as a linear estimator and can only learn the linear part of the sum above; the nonlinear part effectively acts as an unlearnable noise for $\hat\vbeta_1$. Note that only directions $\vbeta_{\star,k}$ with $\tau(g_k) = 1$ (i.e., those with $c_{g_k, 1} \neq 0$) contribute to the linear component, and as a result, a single step of gradient descent cannot learn directions $\vbeta_{\star, k}$ with $\tau(g_k) > 1$. We formalize this in the following theorem.

\begin{proposition}
        \label{eq:alignment_one_step}
        Under the conditions of Section~\ref{sec:conditions}, for any $p\in [M]$, the learned vector $\hat\vbeta_1$ satisfies
        \begin{align*}
        \Big|\hat\vbeta_1^\top \vbeta_{\star, p} - c_{g_{p},1}\Big| = o_{\mathbb{P}}(1).
    \end{align*}
\end{proposition}

This proposition shows that for any component $p \in [M]$ of the target function in \eqref{eq:data_gen} with $\tau(g_p)  = 1$, the learned direction $\hat\vbeta_1$ has a non-trivial alignment with $\vbeta_{\star,k}$.  This alignment of weights with target directions is often referred to as weak recovery (see e.g., \cite{dandibenefits,dandi2023learning,arnaboldi2024repetita,troiani25a}, etc.). Given a nontrivial alignment between the weights and a target direction, the second layer $\bfa$ can be trained using regularized least squares to achieve good generalization performance (see e.g., \cite[Theorem 11]{ba2022high}, \cite[Theorem 4.5]{moniri_atheory2023}, etc.). 

In what follows, we analyze the second step of gradient descent and provide a precise characterization of the updated weight matrix $\bfW_2$ and study the directions that the second update can learn, and the role of the step-size. First, we provide an approximation for the gradient of the loss function evaluated at $(\bfW_1, \bfa_0)$ for a batch of training samples $(\bfX_2, \bfy_2)$.

\begin{theorem}
    \label{thm:second-gradient}
         Under the conditions of Section~\ref{sec:conditions}, the gradient matrix $\nabla_{\mathbf{W}} \widehat{\mathcal{L}}_{\rm Corr}(\mathbf{W}_{1}, \mathbf{a}_0; \bfX_2, \bfy_2)$, can be approximated with the matrix $\bfG$ given by
         \begin{align*}
         \bfG := -\sum_{k = 0}^{L-1} (k+1){c}_{k+1} c_1^k \eta_1^k\; {\bfa_0^{\odot (k+1)}} \hat\vbeta_{2;k}^\top,\quad \text{with}\quad \hat\vbeta_{2;k}:=\frac{1}{\xi_2 n}\bfX_2^\top\parant{\bfy_2 \odot (\bfX_2\hat\vbeta_1)^{\odot k}}
         \end{align*}
     in which $\hat\vbeta_1$ is defied in Proposition~\ref{prop:one-step}, such that 
    \begin{align*}
        \Big\|\nabla_{\mathbf{W}} \widehat{\mathcal{L}}_{\rm Corr}(\mathbf{W}_{1}, \mathbf{a}_0; \bfX_2, \bfy_2) - \bfG\Big\|_{\rm op} = \tilde{O}_{\mathbb{P}}\parant{\frac{1}{\sqrt{N}}}.
    \end{align*}

\end{theorem}

Notably, this theorem shows that the gradient used in the second iteration is no longer rank-one. Hence, unlike the initial update, the model is no longer restricted to learning a single direction. In particular, the gradient matrix aligns with a new set of directions, $\{\hat\vbeta_{2;k}\}_{k \in [L]}$.  In the next theorem, we use this result to  provide a full characterization of the weight matrix $\bfW_2$.

\begin{theorem}
    \label{thm:W2_expansion}
    Assume that the conditions of Section~\ref{sec:conditions} hold with step-sizes $\eta_1 \asymp N^{\alpha_1}$ and $\eta_2 \asymp N^{\alpha_2}$, and define the function 
    $\Lambda:[0,1/2)\times[0,1/2) \to \mathbb{N}$ as
    \begin{align*}
    \Lambda(\alpha_1, \alpha_2) := \min\parant{L-1,\left\lfloor\frac{\alpha_2}{1/2-\alpha_1}\right\rfloor}, \quad \forall \alpha_1, \alpha_2 \in [0,1/2).
    \end{align*}
    The updated weight matrix after the second step of gradient descent update \eqref{eq:iteration}, $\bfW_2$, is given by 
    \begin{align*}
    \bfW_2 =  \bfW_0 &+ c_1  \bfa_0 \parant{\eta_1 \hat\vbeta_1 + \eta_2 \hat\vbeta_{2;0}}^\top  +  \sum_{k = 1}^{\Lambda(\alpha_1, \alpha_2)}  (k+1){c}_{k+1} c_1^k \;\eta_2\eta_1^k\; {\bfa_0^{\odot (k+1)}} {\hat\vbeta_{2;k}}^\top + \bfDelta_2
\end{align*}
where the vectors $\hat\vbeta_{2;k}$ are defined in Theorem~\ref{thm:second-gradient}, and $\|\bfDelta_2\|_{\rm op} = o_{\mathbb{P}}(1)$.
\end{theorem}
In the above approximation expansion for $\bfW_2$, we have $\|\bfW_0\|_{\rm op} = O_{\mathbb{P}}(1)$, $\|c_1\eta_1 \bfa_0 \hat\vbeta_1\|_{\rm op} = O_{\mathbb{P}}(\eta_1) = O_{\mathbb{P}}(N^{\alpha_1})$, $\|c_1\eta_2 \bfa_0 \hat\vbeta_{2;0}\|_{\rm op} = O_{\mathbb{P}}(\eta_2) = O_{\mathbb{P}}(N^{\alpha_2})$. Also, the $k$-th term of the sum has an operator norm $O_{\mathbb{P}}(\eta_2\eta_1^k N^{-k/2}) = O_{\mathbb{P}}(N^{\alpha_2 + k(\alpha_1 - 1/2)})$. Thus, all the $\Lambda$ terms in the sum have operator norm $\Omega_{\mathbb{P}}(1)$. Hence, the matrix $\bfDelta_2$ is negligible compared to the first terms.

This theorem shows that $\bfW_2$ is  a spiked random matrix, with a spectrum consisting of a bulk of singular-values that stick together resulting from the initialization $\bfW_0$ and multiple outliers from the low-rank terms. The number of these outliers is controlled by $\Lambda(\alpha_1, \alpha_2)\in \mathbb{N}$ and depends on the scaling of the step-sizes. Figure~\ref{fig:Lambda} shows $\Lambda(\alpha_1, \alpha_2)$ for different values of $\alpha_1, \alpha_2 \in [0,0.5)$. 

This theorem shows that after the second step of gradient descent update, the weights are in the \texttt{Bulk+Spikes} phase of training from \cite{martin2021implicit}, and that the dynamics closely match the early phase training dynamics empirically characterized in prior work (see also \cite{martin2021predicting, wang2022spectral}).

Using the theorem above, we immediately have the following corollary.
\begin{corollary}
    Under the same conditions as Theorem~\ref{thm:W2_expansion}, the following statements hold:
    \begin{itemize}
        \item If $\alpha_1 + \alpha_2 < 1/2$, we have $\Lambda(\alpha_1, \alpha_2) = 0$ and no new directions will be learned compared to the one-step update, and
        \begin{align*}
            \bfW_2 =  \bfW_0 + c_1 \bfa_0 \left(\frac{\eta_1}{\xi_1 n}\bfX_1^\top\bfy_1 + \frac{\eta_2}{\xi_2 n}\bfX_2^\top\bfy_2 \right)^\top + \bfDelta_3, \quad \text{with} \quad \|\bfDelta_3\|_{\rm op}= o_{\mathbb{P}}(1)
        \end{align*}
        In particular note that this is always the case, regardless of $\alpha_2$, whenever $\alpha_1 = 0$.

        \item When $\alpha_1 \to 0.5^{-}$, we have $\Lambda(\alpha_1, \alpha_2) = L-1$ as long as $\alpha_2 > 0$.
    \end{itemize}
\end{corollary}

Note that in particular, the first statement of the corollary shows that in the reused batch setting, if $\alpha_1  + \alpha_2 <1/2$, the second update is exactly equivalent to a one-step of update with step-size $\eta_1+\eta_2$. This extends the finding in \cite{ba2022high, wang2022spectral} that proved a similar result when $\alpha_1, \alpha_2=0$.

Moreover, the second statement shows that if $\alpha_1 \to 0.5^{-}$ (as long as $\alpha_2 > 0$), the number of learned directions will proliferate. Note that in this paper, we do not cover the case with $\alpha_1$ or $\alpha_2$ equal to $0.5$. When $\alpha_1 = 0.5$, the error term $\bfDelta_1$ in Proposition~\ref{prop:one-step} will have $\|\bfDelta_1\|_{\rm op}=\Theta_{\mathbb{P}}(1)$. Although in that case $\|\bfDelta_1\|_{\rm op} \asymp \|\bfW_0\|_{\rm op}$ and $\bfDelta_1$ can no longer be neglected, it can be be shown that $\bfDelta_1$ essentially equivalent to an independent noise, with the effect of rescaling the variance of the untrained weight $\bfW_0$ (see \cite[Section A.4]{cuiasymptotics}). A similar phenomenon will also occur for the second gradient step when $\alpha_2 = 0.5$ with $\bfDelta_2$ becoming comparable to $\bfW_0$ in operator norm. We leave the analysis of the boundary cases as future work.

\begin{figure}
    \centering
    \includegraphics[width=0.4\linewidth]{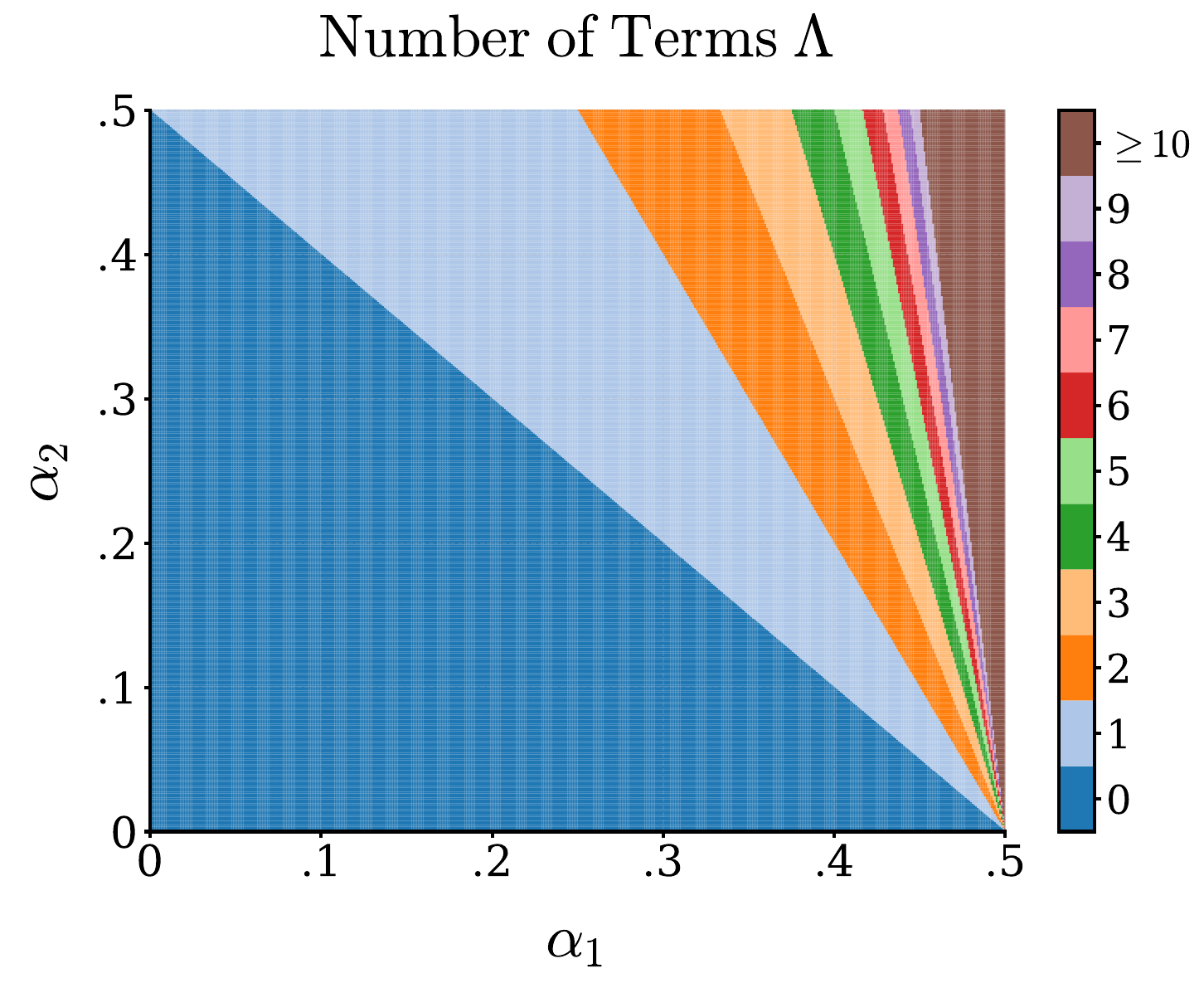}
    \caption{The number of different directions $\Lambda(\alpha_1, \alpha_2)$ learned by the second step of gradient descent, as a function of $\alpha_1, \alpha_2 \in [0, 0.5)$.}
    \label{fig:Lambda}
\end{figure}

\section{Analysis of the Learned Directions}
Now let's turn back to the learned directions  $\hat{\vbeta}_{2;k}$. Recall that the direction learned by the first step of gradient descent $\hat{\vbeta}_1 = (\xi_1 n)^{-1}\bfX_1^\top \mathbf{y}_1$ is a linear estimator, and thus $\hat{\mathbf{y}}_2 := \bfX_2\hat{\vbeta}_1$ is the prediction of this estimator for covariates $\bfX_2$. The vector $\hat{\vbeta}_{2;k} = (\xi_2 n)^{-1} \bfX_2^\top (\mathbf{y}_2 \odot \hat{\mathbf{y}}_2^{\odot k})$ is  a nonlinear estimator for predicting the labels $\mathbf{y}_2$ from $\bfX_2$. This formulation underscores the inherent difference between the reused and fresh batch settings of Section~\ref{sec:problem_setting}: in the reused batch setting, $\hat{\mathbf{y}}_2$ is the prediction for the same samples used to construct $\hat{\vbeta}_1$; conversely, in the fresh batch setting, $\hat{\mathbf{y}}_2$ is the prediction for a fresh, independent set of samples that were not used in the construction of $\hat{\vbeta}_1$.
In what follows, we  study the alignments of the learned directions $\{\hat\vbeta_{2;k}\}_{k = 1}^{\Lambda(\alpha_1, \alpha_2)}$ with the target directions $\{\vbeta_{\star, k}\}_{k = 1}^{M}$. We use it to characterize the differences between the reused batch and the fresh batch settings.

\begin{figure}
    \centering
    \includegraphics[width=\linewidth]{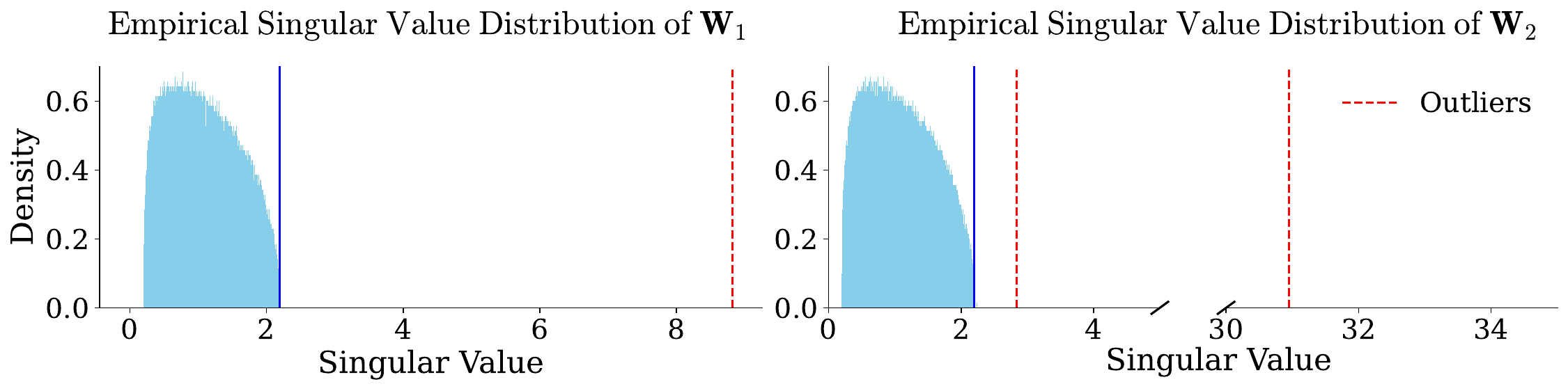}
    \caption{The histogram of the singular values of $\bfW_1$ and $\bfW_2$ with $\alpha_2 = 0.4$, $\alpha_1 = 0.3$. In this setup, we consider the reused batch setting and set $M=1$ with $g_1(z) = H_3(z)$, $\sigma(z) = \tanh(z)$, $n = 40\times 10^3, \dx = 14\times 10^3$ and $N = 20\times 10^3$.  The vertical dashed lines correspond to the outliers and the vertical blue line is the top non-outlying singular value. As predicted by Theorem~\ref{thm:W2_expansion}, the spectrum of $\bfW_2$ contains two outlying singular vector, whereas the spectrum of $\bfW_1$ contains a single outlier. See Section~\ref{sec:experiments} for more discussions.}
    \label{fig:simulation_histogram} 
\end{figure}
First, we consider the reused batch setting and compute the inner product between each target direction $\vbeta_{\star, p}$ with all directions learned by the second step $\hat\vbeta_{2;k}$.
\begin{theorem}
\label{thm:reuse}
Assume that the conditions of Section~\ref{sec:conditions} hold. In the \textbf{reused batch} setting, for any $p \in [M]$ and $q \in [\Lambda(\alpha_1, \alpha_2)]$ we have
\begin{align*}
    \vbeta_{\star, p}^\top\, \hat\vbeta_{2;q} \to_{\mathbb{P}} \Ex\left[z_p y \parant{\phi y + \sum_{k = 1}^{M}c_{g_k, 1}z_k + G}^q\right]
\end{align*}
where $z_1, \dots, z_M \stackrel{\mathrm{i.i.d.}}{\sim} \normal(0,1)$, $\ep \sim \normal(0, \sigma_\ep^2)$ are independent random variables, and $y = \ep + \sum_{k = 1}^{M}g_k(z_k) $. Moreover $G \sim \normal(0, \phi \, \Ex[y^2])$ independent of other randomness.
\end{theorem}
Note that when $q=0$, the above theorem simplifies to $\vbeta_{\star, p}^\top\, \hat\vbeta_{2;0}  \to \Ex[z_p y] = c_{g_p, 1}$ which recovers  Proposition~\ref{eq:alignment_one_step}. The above theorem shows that  $\vbeta_{\star, p}^\top\, \hat\vbeta_{2;q}$ can be non-zero even  if $c_{g_p, 1} = 0$. We demonstrate this in the following simple example.
\begin{example}
 Consider the batch reuse setting and assume that $M= 1$, $\sigma_\ep = 0$, and $g_1(x) = H_3(x)$. We have  $\vbeta_{\star, 1}^\top\, \hat\vbeta_{2;q} \to \Ex[z H_3(z) \parant{\phi H_3(z) + G}^q]$ that is non-zero for $q = 2$. This means that $\bfW_2$ will have a non-trivial alignment with $\vbeta_{\star,1}$ if the step-sizes satisfy $\Lambda(\alpha_1, \alpha_2) \geq 2$.   
\end{example}

This example demonstrates that in the reused batch setting, with proper step-sizes, the second update does can in fact make $\bfW_2$ aligned to directions corresponding to an information-exponent greater than one. This is in line with a similar finding in the narrow-width regime by \cite{dandibenefits,lee2024neural} showing that gradient descent with reused batches can learn functions with information exponent exceeding one. Now, let's move to the fresh batch setting.

\begin{theorem}
\label{thm:fresh}
Assume that the conditions of Section~\ref{sec:conditions} hold. In the \textbf{fresh batch} setting,  for any $p \in [M]$ and $q \in [\Lambda(\alpha_1, \alpha_2)]$ we have
\begin{align*}
    \vbeta_{\star, p}^\top\, \hat\vbeta_{2;q} \to_{\mathbb{P}} \Ex\left[z_p y \parant{ \sum_{k = 1}^{M}c_{g_k, 1}z_k + G}^q\right]
\end{align*}
where $z_1, \dots, z_M \stackrel{\mathrm{i.i.d.}}{\sim} \normal(0,1)$, $\ep \sim \normal(0, \sigma_\ep^2)$ are independent random variables, and $y = \ep + \sum_{k = 1}^{M}g_k(z_k) $. Moreover $G \sim \normal(0, \phi \, \Ex[y^2] /\xi_2)$ independent of other randomness.  Consequently,  for any  direction $\vbeta_{\star, p}$ with $c_{g_p, 1}=0$ we have  $\vbeta_{\star, p}^\top\, \hat\vbeta_{2;q} \to 0$ for all $q$.
\end{theorem}
This limit is fundamentally different from the limit in Theorem~\ref{thm:reuse}. In particular, it shows that in the fresh batch setting, the two-step updated weight matrix $\bfW_2$ never aligns with a target direction that has information exponent greater than one. 

\begin{figure}
    \centering
    \underline{\small Reused Batch Setting}
    \includegraphics[width=0.95\linewidth]{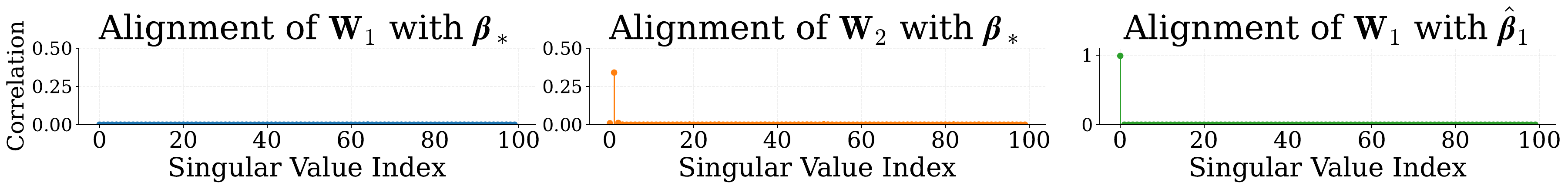}\\[0.1cm]
    \underline{\small Fresh Batch Setting}
    \includegraphics[width=0.95\linewidth]{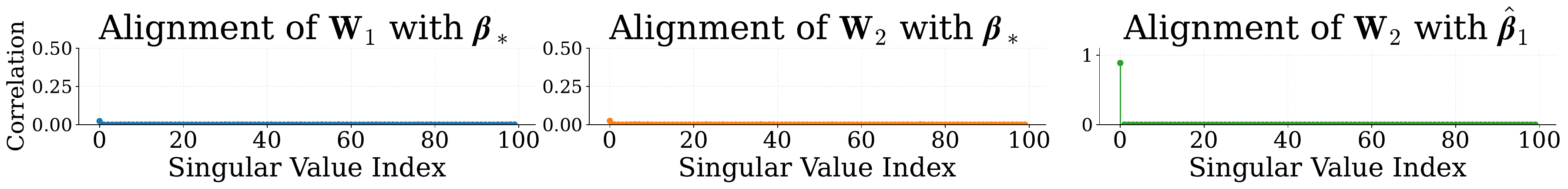}
    \caption{The alignment of right singular vectors of the weight matrix, with the target direction in the reused batch and the fresh batch settings. We set $M=1$ with $\bfy = H_3(\bfX\vbeta_\star)$ for a unit norm $\vbeta_\star \in \R^{\dx}$, $\sigma(z) = \tanh(z)$, $n = 40\times 10^3, \dx = 14\times 10^3$, $N = 20\times 10^3$, with $\alpha_2 = 0.4$, $\alpha_1 = 0.3$. As predicted by our theory, in the reused batch setting, the \textit{second} top singular vector aligns with $\vbeta_\star$, whereas there is no alignment in the fresh batch setting. See Section~\ref{sec:experiments} for more discussions.} 
    \vspace{-0.5cm}
    \label{fig:simulation_alignments}
\end{figure}
\section{Numerical Simulations}
\label{sec:experiments}
\paragraph{Simulation 1.} In this numerical simulation, we set $M=1$, $\sigma_\ep = 0$ and consider $\bfy = H_3(\bfX\vbeta_\star)$ with $\vbeta_\star \sim \normal(0, \dx^{-1}\bfI_{\dx})$. We set $n = 40\times 10^3, \dx = 14\times 10^3$, $N = 20\times 10^3$ and update the first layer weights using $\eta_1 = N^{0.3}$ and $\eta_2 = N^{0.4}$. In Figure~\ref{fig:simulation_histogram}, we plot the histogram of the singular values of $\bfW_1$ and $\bfW_2$ for the batch reusse setting. 
We observe that as predicted in Theorem~\ref{thm:W2_expansion}, for $\bfW_2$ two outliers emerge in the spectrum, whereas for $\bfW_1$, there is only one outlier. This closely matches the theoretical finding in Theorem~\ref{thm:W2_expansion}.

\begin{figure}
    \centering    \includegraphics[width=0.4\linewidth]{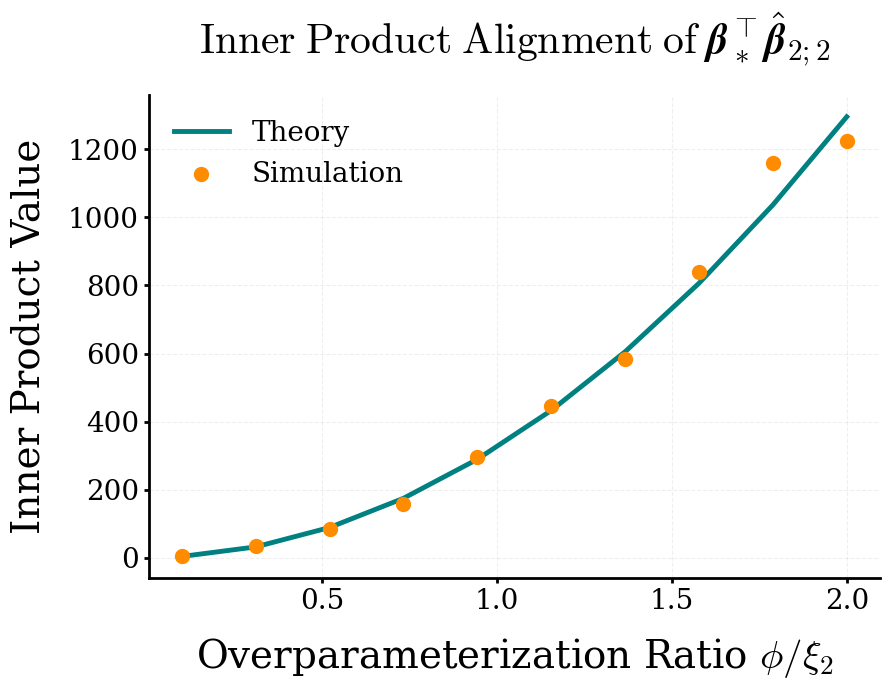}
    \caption{Inner Product Alignment $\vbeta_*^\top \hat{\vbeta}_{2;2}$ as a function of $\dx/{\xi_2n}$.}
    \label{fig:inner_product}
\end{figure}
We then compute the correlation between the vector $\vbeta_\star$ with the right singular vectors of the updated weight matrices $\bfW_1$ and $\bfW_2$, once for the reused batch and once for the fresh batch setting. This is shown in Figure~\ref{fig:simulation_alignments}. We observe that in the reused batch setting, the \textit{second} top singular-vector aligns with the true direction $\vbeta_\star$, as predicted in Theorems ~\ref{thm:W2_expansion} and \ref{thm:reuse}. However, the weights do not align with $\vbeta_\star$ in the fresh batch setting (Theorem~\ref{thm:fresh}). Moreover, as predicted, we observe that the top singular vector in both settings is given by the vector $\hat\vbeta_1$.

\paragraph{Simulation 2.} In this setting, we consider $n = 6\times 10^3$ and vary $\dx$ between $600$ and $12\times 10^3$. We set $M=1$ with $ \bfy = H_3(\bfX\vbeta_\star)$. We compute $\vbeta_\star^\top\hat\vbeta_{2;2}$ and average it for each values of $\dx$ over $100$ trials. We compare the simulation results with the theoretical predictions of Theorem~\ref{thm:reuse}. This is shown in Figure~\ref{fig:simulation_alignments}. We observe that even for moderately large $n, \dx$, the theoretical predictions provide a good approximation of the simulations.

\section{Conclusion}
In this paper, we establish a sharp spectral characterization of the first-layer weight matrix in two-layer neural networks after two steps of gradient descent. We show that the spectrum of the updated weights consists of a bulk of eigenvalues from initialization and multiple outliers, which we fully characterize for different scaling regimes of the step-sizes. Our analysis demonstrates that the second update allows the model to capture multiple relevant directions, effectively expanding the model's capacity to learn beyond single-index functions.  These findings provide a rigorous mathematical framework for understanding the emergence of structured representations during the early stages of optimization in the linear-width regime.

\section{Acknowledgments}
Behrad Moniri gratefully acknowledges the gift from AWS AI to Penn Engineering’s ASSET Center
for Trustworthy AI. The work of Behrad Moniri and Hamed Hassani is supported by The Institute
for Learning-enabled Optimization at Scale (TILOS), under award number NSF-CCF-2112665.

{\small
\bibliography{full_references}
\bibliographystyle{alpha}
}

\newpage

\section{Useful Lemmas}
\begin{lemma}
    \label{lemma:rank1-hadamard}
    For any $\bfu \in \R^{d_u}$ and $\bfv \in \R^{d_v}$ and $\bfA \in \R^{d_u \times d_v}$, we have
    \begin{align*}
        (\bfu\bfv^\top) \odot \bfA = \diag(\bfu) \bfA \diag(\bfv).
    \end{align*}
\end{lemma}
\begin{proof}
    For any indices $(i, j)$, the $(i, j)$-th entry of the left-hand side is$$[(\mathbf{u}\mathbf{v}^\top) \odot \mathbf{A}]_{ij} = (\mathbf{u}\mathbf{v}^\top)_{ij} A_{ij} = u_i v_j A_{ij}.$$
    Similarly, the $(i, j)$-th entry of the right-hand side is$$[\text{diag}(\mathbf{u}) \mathbf{A}\, \text{diag}(\mathbf{v})]_{ij} = \sum_{k,l} (\text{diag}(\mathbf{u}))_{ik} A_{kl} (\text{diag}(\mathbf{v}))_{lj}.$$
    The sums collapse to $u_i A_{ij} v_j$, finishing the proof.
\end{proof}
\begin{lemma}[Taylor expansion of Hermite Polynomials] 
\label{lemma:Hermite-of-Sum}
For any $k \in \mathbb{N}_0$ and $x, y \in \mathbb{R}$,
\begin{align*}
H_k(x+y)=\sum_{j=0}^k\binom{k}{j} x^j H_{k-j}(y).
\end{align*}
\end{lemma}
\begin{proof}
    Note that $\frac{d}{d x} H_k(x)=k H_{k-1}(x)$ \cite[Equation 22.8.8]{abramowitz1968handbook} and thus $\frac{d^j}{d x^j} H_k(x)=\frac{k!}{(k-j)!} H_{k-j}(x)$. By Taylor expanding $H_k(x+y)$ at $y$, we find
    $$
    H_k(x+y)=\sum_{j=0}^k \frac{x^j}{j!} \frac{d^j}{d y^j} H_k(y)=\sum_{j=0}^k\binom{k}{j} x^j H_{k-j}(y),
    $$
    concluding the proof.
\end{proof}

\begin{lemma}
    \label{lemma:easy_norms}
    Each of the following statements  hold with probability $1 - o(1)$:
    \begin{enumerate}
        \item [(a)] $\|\bfy_1\|_{\infty} = \tilde\Theta(1)$\; (similarly for $\|\bfy_2\|_\infty$).
        \item [(b)] $\|\bfa_0^{\odot p}\|_{\infty} = \tilde\Theta({N^{-\frac{p}{2}}})$
        \item [(c)] $\|\bfa_0^{\odot p}\|_2 = \Theta(N^{\frac{1-p}{2}})$
        \item [(d)] $\|\bfX_1\|_{\rm op}=\Theta(\sqrt{N})$ (similarly for $\|\bfX_2\|_{\rm op}$).
        \item [(e)] $\Big\|\bfy_2 \odot (\bfX_2\hat\vbeta_1)^{\odot j }\Big\|_{\infty} = \tilde\Theta(1)$
    \end{enumerate}
\end{lemma}
\begin{proof}
    Here, we prove each statement separately.
    \begin{enumerate}
        \item [(a)] Note that $\bfy_1 = \boldsymbol{\ep} +  \sum_{k = 1}^{M}g_k(\bfX_1\vbeta_{\star, k})$.  Using the standard (sub-)gaussian maximal inequality  \cite[Proposition 2.7.6]{vershynin2018high}, we have $\max_{i \in[n]}\bracket{\ep_i} = \tilde\Theta(1)$ with probability $1-o(1)$. For the signal term, recall that $\|\vbeta_{\star, k}\| = 1$, thus the entries of $\bfX_1 \vbeta_{\star, k}$ are independent and distributed according to $\normal(0,1)$. Thus, using standard Gaussian Maximal Inequalities, with probability $1-o(1)$ we have
        $\max_{i \in[n]}\bracket{|\bfx_i^\top \vbeta_{\star, k}|} = \tilde\Theta(1)$.  Putting these together proves part (a).
        \item [(b)]  Similar to part (a), this part follows from the standard gaussian maximal inequality, by recalling that each element $i \in [N]$ of $[\bfa_0]_i \sim \normal(0, 1/N)$. 
        \item [(c)] Let $a_{i}\sim \normal(0, 1/N)$ for $i\in [N]$ be the $i$-th entry of $\bfa_0\in \R^N$. We have
        \begin{align*}
            \|\bfa_0^{\odot p}\|_2^2 = \sum_{i = 1}^{n} a_{i}^{2p}.
        \end{align*}
        Thus, we have $\Ex\|\bfa_0^{\odot p}\|_2^2 = (2p-1)!! \cdot N^{1-p}$. Part (c) follows by applying the sub-Weibull concentration inequality  \cite{vladimirova2020sub} to prove the concentration of $\|\bfa_0^{\odot p}\|_2^2$ around $\Ex\|\bfa_0^{\odot p}\|_2^2$.
        \item [(d)] Part (d) is standard, see e.g., \cite[Theorem 4.4.5]{vershynin2018high}.
        \item [(e)] To prove this part, first note that from part (a), $\|\bfy_2\|_\infty = \tilde{\Theta}(1)$. Also, using an expansion of $\bfy$ similar to the proof of part (a), we have $\|\hat\vbeta_1\|_2 = \|\frac{1}{\xi_1 n}\bfX_1^\top \bfy_1\|_2 = \Theta_{\mathbb{P}}(1)$. Putting these together, will conclude the proof.
    \end{enumerate}
\end{proof}

\begin{lemma}
    \label{lemma:hermite_norm}
    Let $k \geq 1$ and consider $H_k(\bfW_0 \bfX_1^\top) \in \R^{N \times \xi_1 n}$. With probability $1 - o(1)$ we have
    \begin{align*}
    \|H_k(\bfW_0 \bfX_1^\top)\|_{\rm op} = \tilde{O}(\sqrt{N}).
    \end{align*}
    Similarly $\|H_k(\bfW_0 \bfX_2^\top)\|_{\rm op} = \tilde{O}(\sqrt{N}).$
\end{lemma}
\begin{proof}
    First, note that by the orthogonality of Hermite polynomials and \cite[Chapter 11.2]{o2014analysis}, we can write
    \begin{align*}
        \Ex_{\bfx \sim \normal(\mathbf{0}, \bfI_{\dx})}\bracket{H_k(\bfW_0\bfx)H_k(\bfW_0\bfx)^\top} = k! (\bfW_0 \bfW_0^\top)^{\circ k}.
    \end{align*}
    Next, using \cite[Theorem 5.39]{vershynin2010introduction} and  \cite[Corollary A.21]{bai2010spectral}, we have 
    $$\|k! (\bfW_0 \bfW_0^\top)^{\circ k}\|_{\rm op} \leq C k!$$
    for some universal constant $C>0$. Now, to upper bound the operator norm of $H_k(\bfX\bfW_0^\top)$, note that
    \begin{align*}
        \frac{1}{n}\|H_k(\bfX\bfW_0^\top)\|_{\rm op}^2 &= \frac{1}{n}\|H_k(\bfX\bfW_0)^\top H_k(\bfX\bfW_0)\|_{\rm op}\\
        &\leq \left\|\frac{1}{n}H_k(\bfX\bfW_0)^\top H_k(\bfX\bfW_0) - k! (\bfW_0\bfW_0^\top)^{\odot k}\right\|_{\rm op} + \left\|k! (\bfW_0\bfW_0^\top)^{\odot k}\right\|_{\rm op}.
    \end{align*}
    Hence, we have
    \begin{align*}
        \mathbb{P}\bracket{\|H_k(\bfX\bfW_0^\top)\|_{\rm op}\geq t} &= \mathbb{P}\bracket{\frac{1}{n}\|H_k(\bfX\bfW_0^\top)\|_{\rm op}^2\geq t^2/n}\\
        &\leq \mathbb{P}\bracket{\left\|\frac{1}{n}H_k(\bfX\bfW_0)^\top H_k(\bfX\bfW_0) - k! (\bfW_0\bfW_0^\top)^{\odot k}\right\|_{\rm op} \geq t^2/n - Ck!}\\
        &\leq \frac{\Ex\bracket{\left\|\frac{1}{n}H_k(\bfX\bfW_0)^\top H_k(\bfX\bfW_0) - k! (\bfW_0\bfW_0^\top)^{\odot k}\right\|_{\rm op}}}{t^2/n - Ck!}
    \end{align*}
    where in the last line, we used the Markov inequality. We can use \cite[Theorem 5.8]{vershynin2010introduction} to upper bound the numerator as follows:
    \begin{align}
        \label{eq:big_expected}
        \Ex\bracket{\left\|\frac{1}{n}H_k(\bfX\bfW_0)^\top H_k(\bfX\bfW_0) - k! (\bfW_0\bfW_0^\top)^{\odot k}\right\|_{\rm op}} \leq  \max\parant{(Ck!)^{1/2}\delta, \delta^2}
    \end{align}
    where $\delta$ is given by
    \begin{align*}
        \delta = C \sqrt{\frac{m\log \min(N, n)}{N}}
    \end{align*}
    with $m = \Ex \max_{i \in [n]}\|H_k(\bfW_0 \bfx_i)\|^2_2$. Note that there exists constants $C_1, C_2$ such that 
    $$m \leq C_1 \Ex\max_{i \in [n]}\|(\bfW_0 \bfx_i)^{\odot k}\|^2_2 \leq C_1 \Ex\max_{i \in [n]} \|\bfW_0 \bfx_i\|^{2k}_2 \leq C_2 \Ex\max_{i \in [n]} \|\bfx_i\|^{2k}_2$$ 
    because $\|\bfu\odot \bfv\|_2\leq \|\bfu\|_2\cdot\|\bfv\|_2$ for any $\bfu, \bfv$. 
    Note that $\|\bfx_i\|_2^{2k}/N$ for $i\in [n]$ are sub-Weibull (also called $\alpha$-subexponential) random variables with tail parameter $1/(2k)$ \cite{vladimirova2020sub}. Using the maximal inequality for sub-Weibull random variables \cite[Proposition A.6 and Remark A.1]{kuchibhotla2022moving}, it follows that there is ${C}_3$ such that
    $\Ex  \max_{i=1}^n \|\bfx_i\|^{2k}/N \le C_3(\log n)^{2k}$. Hence, $m \leq C_2C_3 N (\log n)^{2k}$. Thus,
    \begin{align*}
        \delta = C \sqrt{\frac{ C_2C_3 N (\log n)^{2k}\log \min(N, n)}{N}} \leq C_4 \parant{\log(N)}^{\frac{2k+1}{2}}
    \end{align*}
    Plugging this back into \eqref{eq:big_expected} and setting $t = C' \sqrt{N} \cdot \parant{\log(N)}^{\frac{k+1}{2}}$ for large enough $C'$ yields
    \begin{align*}
        \mathbb{P}\bracket{\|H_k(\bfX\bfW_0^\top)\|_{\rm op}\geq C' \sqrt{N} \cdot \parant{\log(N)}^{\frac{k+1}{2}}} \to 0.
    \end{align*}
    Thus, with probability $1-o(1)$, we have $\|H_k(\bfX\bfW_0^\top)\|_{\rm op} = \tilde{O}(\sqrt{N})$, proving the theorem.
\end{proof}
\section{Proof of Proposition~\ref{prop:one-step}}
\begin{proof}
At initialization $\bfW_0, \bfa_0$, the gradient of the loss function \eqref{eq:corr_loss_gradient} is given by 
\begin{align*}
    \nabla_\bfW \widehat{\mathcal{L}}_{\rm Corr}(\bfW_0, \bfa_0; \bfX_1, \bfy_1) &= -\frac{1}{\xi_1{n}} \left[(\bfa_0\bfy_1^\top)\odot \sigma'(\bfW_0\bfX_1^\top)\right]\bfX_1.
\end{align*}
Using the Hermite expansion of $\sigma$, we  write $\sigma(z) = c_1 z + \sigma_\perp(z)$ where $\Ex_{z\sim \normal(0,1)}[z \sigma_\perp(z)] = 0$. With this expansion, we have
\begin{align*}
    \nabla_\bfW \widehat{\mathcal{L}}_{\rm Corr}(\bfW_0, \bfa_0; \bfX_1, \bfy_1) &= -c_1 \bfa_0 \left(\frac{\bfX_1^\top \bfy_1}{\xi_1{n}}\right)^\top -\frac{1}{\xi_1{n}} \left[(\bfa_0\bfy_1^\top)\odot \sigma_{\perp}'(\bfW_0\bfX_1^\top)\right]\bfX_1\\
    &= -c_1 \bfa_0 \left(\frac{\bfX_1^\top \bfy_1}{\xi_1{n}}\right)^\top - \bfE.
\end{align*}
Note that $\Ex_{z \sim \normal (0,1)}[\sigma_\perp'(z)] = 0$. Using Lemma~\ref{lemma:rank1-hadamard}, we can write
\begin{align*}
    \bfE = \frac{1}{\xi_1{n}} \left[\diag(\bfa_0) \, \sigma_\perp'(\bfW_0\bfX_1^\top)\diag(\bfy_1)\right]\bfX_1.
\end{align*}
We will next analyze the operator norm of the matrix $\bfE$. To do this, note that using Lemma~\ref{lemma:easy_norms}:
\begin{align*}
    &\|\diag(\bfy_1)\|_{\rm op} = \|\bfy_1\|_\infty = \tilde\Theta_{\mathbb{P}}(1),\quad \|\diag(\bfa_0)\|_{\rm op} = \|\bfa_0\|_\infty = \tilde\Theta_{\mathbb{P}}\parant{\frac{1}{\sqrt{N}}},\\[0.1cm]
    &\text{and} \quad \|\bfX_1\|_{\rm op} = \Theta_{\mathbb{P}}(\sqrt{N}).
\end{align*}
Moreover, using the property that  $\frac{d}{d x} H_k(x)=k H_{k-1}(x)$ \cite[Equation 22.8.8]{abramowitz1968handbook}, we can write
\begin{align*}
\|\sigma_\perp'(\bfW_0\bfX_1^\top)\|_{\rm op}  &= \left\|\sum_{k = 1}^{L-1}c_{k+1}(k+1)H_k(\bfW_0\bfX_1^\top)\right\|_{\rm op}\\[0.2cm]
&\leq \sum_{k = 1}^{L-1}|c_{k+1}|(k+1)\left\|H_k(\bfW_0\bfX_1^\top)\right\|_{\rm op} = \tilde\Theta_{\mathbb{P}}(\sqrt{N}),
\end{align*}
where we have used Lemma~\ref{lemma:hermite_norm}. Hence, we can use the sub-multiplicativity property of operator norm to write
\begin{align*}
    \|\bfE\|_{\rm op} = \tilde{O}_{\mathbb{P}}\parant{\frac{1}{\sqrt{N}}}.
\end{align*}
Next, recall that $\eta_1 \asymp N^{\alpha_1}$ with $\alpha_1 < 1/2$. Thus
\begin{align*}
    \bfW_1 = \bfW_0 + c_1 \eta_1 \bfa_0  \left(\frac{\bfX_1^\top \bfy_1}{\xi_1{n}}\right)^\top + \eta_1 \bfE
\end{align*}
where $\|\eta_1\bfE\|_{\rm op} = o_{\mathbb{P}}(1)$, completing the proof.
\end{proof}

\section{Proof of Proposition~\ref{eq:alignment_one_step}}
\begin{proof}
Recall that the vector $\hat\vbeta_1$ can be written as
\begin{align*}
    \hat\vbeta_1 = \frac{1}{\xi_1 n} \bfX_1^\top \bfy_1 =  \frac{1}{\xi_1 n} \bfX_1^\top \boldsymbol{\ep} + \frac{1}{\xi_1 n} \sum_{k = 1}^{M}\bfX_1^\top{g_k(\bfX_1\vbeta_{\star, k})}.
\end{align*}
where $\boldsymbol{\ep} \in \R^{\xi_1 n}$ is the noise from \eqref{eq:data_gen}. Let $\bfX_1 = [\bfx_1, \dots, \bfx_{n_1}]^\top$ and denote $n_1 := \xi_1 n$. Expanding the matrix product above, we have
\begin{align*}
    \hat\vbeta_1 = {\frac{1}{n_1} \sum_{i = 1}^{n_1}\sum_{k = 1}^{M} \bfx_{i}\, g_k(\bfx_{i}^\top \vbeta_{\star, k})} + \frac{1}{n_1}\sum_{i = 1}^{n_1} \bfx_{i}\ep_i
\end{align*}
Let $p \in [M]$. We have
\begin{align*}
    \hat\vbeta_1^\top \vbeta_{\star,p}= {\frac{1}{n_1} \sum_{i = 1}^{n_1} \sum_{k = 1}^{M}(\bfx_{i}^\top\vbeta_{\star,p})\, g_k(\bfx_{i}^\top \vbeta_{\star, k}) + \frac{1}{n_1}\sum_{i = 1}^{n_1} (\bfx_{i}^\top\vbeta_{\star,p})\ep_i}.
\end{align*}
We will now analyze each of these two terms separately.
\paragraph{First Sum.} This sum can be written as
\begin{align*}
    \frac{1}{n_1} \sum_{i = 1}^{n_1} \sum_{k = 1}^{M}(\bfx_{i}^\top\vbeta_{\star,p})\, &g_k(\bfx_{i}^\top \vbeta_{\star, k})= \frac{1}{n_1} \sum_{i = 1}^{n_1} (\bfx_{i}^\top\vbeta_{\star,p})\, g_p(\bfx_{i}^\top \vbeta_{\star, p}) + \frac{1}{n_1} \sum_{i = 1}^{n_1} \sum_{k \neq p}(\bfx_{i}^\top\vbeta_{\star,p})\, g_k(\bfx_{i}^\top \vbeta_{\star, k}).
\end{align*}
First, consider the first sum on the right-hand side. Note that $\bfx_{i}^\top\vbeta_{\star,p} \sim \normal(0, 1)$. Hence, by the weak law of large number, we have
\begin{align*}
    \frac{1}{n_1}\sum_{i = 1}^{n_1}(\bfx_{i}^\top\vbeta_{\star,p})\, g_p(\bfx_{i}^\top \vbeta_{\star, p}) \to_{\mathbb{P}} \Ex_{z \sim \normal(0,1)}[z g_p(z)] = c_{g_p, 1}.
\end{align*}
For sum corresponding to $k\neq p$, again we have $\bfx_{i}^\top\vbeta_{\star,p} \sim \normal(0,1)$ and $\bfx_{i}^\top\vbeta_{\star,k} \sim \normal(0,1)$. However, because $\vbeta_{\star p}$ and $\vbeta_{\star k}$ are orthogonal for $k \neq p$, we have
\begin{align*}
    \Ex\bracket{(\bfx_{i}^\top\vbeta_{\star,p})(\bfx_{i}^\top\vbeta_{\star,k})} = \vbeta_{\star, p}^\top \vbeta_{\star, k} = 0.
\end{align*}
As a result, using the weak law of large numbers, we have
\begin{align*}
    \frac{1}{n_1} \sum_{i = 1}^{n_1} \sum_{k \neq p}(\bfx_{i}^\top\vbeta_{\star,p})\, g_k(\bfx_{i}^\top \vbeta_{\star, k}) \to_{\mathbb{P}} 0.
\end{align*}

\paragraph{Second Sum.} Note that $\ep_i \sim \normal(0, \sigma_\ep^2)$ and is independent of other randomness. Hence, using the weak law of large numbers, we have
\begin{align*}
    \frac{1}{n_1}\sum_{i = 1}^{n_1}\sum_{k = 1}^{M}(\bfx_{i}^\top\vbeta_{\star,p})\, \ep_i \to_{\mathbb{P}} 0.
\end{align*}

\paragraph{Putting Everything Together.} As a result, we can put everything together and conclude
\begin{align*}
     \hat\vbeta_1^\top \vbeta_{\star,p} \to_{\mathbb{P}} c_{g_p,1}
\end{align*}
completing the proof.
\end{proof}

\section{Proof of Theorem~\ref{thm:second-gradient}}
\begin{proof}
The wight matrix $\bfW_1$ after the  first step of update is characterized in Proposition~\ref{prop:one-step} as
\begin{align}
    \label{eq:first_step_weight}
    \bfW_1 = \bfW_0 + c_1 \eta_1 \bfa_0 \hat\vbeta_1^\top + \bfDelta_1
\end{align}
where $\|\bfDelta_1\|_{\rm op} = {o}_{\mathbb{P}}(1)$. For simplicity, denote $\widetilde\bfW_0 = \bfW_0 + \bfDelta_1$. Now, consider the gradient matrix evaluated at $(\bfW_1, \bfa_0)$ for a training batch $(\bfX_2, \bfy_2) \in \R^{\xi_2 n \times \dx} \times \R^{\xi_2 n}$:
\begin{align*}
    \nabla_\bfW \widehat{\mathcal{L}}_{\rm Corr}(\bfW_1, \bfa_0;\bfX_2, \bfy_2) &= -\frac{1}{\xi_2{n}}\left[(\bfa_0{\bfy_2}^\top)\odot \sigma'(\bfW_1\bfX_2^\top)\right]\bfX_2.
\end{align*}
To simplify the gradient, we can expand $\sigma'$ in the Hermite basis and use the fact that $\frac{d}{dx}H_k(x) = k H_{k-1}(x)$ \cite[Equation 22.8.8]{abramowitz1968handbook} as follows
 \begin{align*}
    \sigma'(\bfW_1\bfX_2^\top) &= \sigma'\left(\widetilde\bfW_0\bfX_2^\top + \eta_1 c_1 \bfa_0(\bfX_2\hat\vbeta_1)^\top\right)\\
    &=\sum_{k = 0}^{L-1} (k+1){c}_{k+1} H_k\left(\widetilde\bfW_0\bfX_2^\top+ \eta_1 c_1 \bfa_0 (\bfX_2\hat\vbeta_1)^\top\right)
\end{align*}
Next, we  expand each Hermite function using Lemma~\ref{lemma:Hermite-of-Sum} to get
\begin{align*}
    \sigma'(\bfW_1\bfX_2^\top)&=\sum_{k = 0}^{L-1} (k+1){c}_{k+1} \sum_{j = 0}^{k} {k \choose j} c_1^j \eta_1^j \; \left(\bfa_0^{\odot j} (\bfX_2\hat\vbeta_1)^{\odot j \top}\right) \odot H_{k - j}\left(\widetilde\bfW_0\bfX_2^\top\right).
\end{align*}
In the inner sum, the terms $j = 0$ and $j = k$ are fundamentally different from other terms.  Hence, we split the double sum as follows:
\begin{align*}
    \sigma'(\bfW_1\bfX_2^\top)&= \left[\sum_{k = 0}^{L-1}(k+1){c}_{k+1} H_k\left(\widetilde\bfW_0\bfX_2^\top\right)\right] + \left[\sum_{k = 1}^{L-1} (k+1){c}_{k+1} c_1^k \eta_1^k\; \bfa_0^{\odot k}\,(\bfX_2\hat\vbeta)^{\odot k \top}\right]\\
    &\hspace{1cm}+\sum_{k = 2}^{L-1} \sum_{j = 1}^{k-1}{k \choose j}(k+1){c}_{k+1} c_1^j \eta_1^j \; \left(\bfa_0^{\odot j} (\bfX_2\hat\vbeta_1)^{\odot j \top}\right) \odot H_{k - j}\left(\widetilde\bfW_0\bfX_2^\top\right)
\end{align*}
The first term is the Hermite expansion of $\sigma'(\widetilde\bfW_0\bfX_2^\top$). Thus, gradient is given by
\begin{align}
    \label{eq:grad_new}
    \nabla_\bfW \widehat{\mathcal{L}}_{\rm Corr}(\bfW_1, \bfa_0; \bfX_2, \bfy_2) = \bfT_1 +\bfT_2 + \bfT_3
\end{align}
in which $\bfT_1, \bfT_2, \bfT_3$ are 
\begin{align*}
    &\bfT_1 = -\frac{1}{\xi_2{n}}\bracket{\parant{\bfa_0\bfy_2^\top} \odot {\sigma'\parant{\widetilde{\bfW}_0\bfX_2^\top}}}\bfX_2,\\[0.2cm]
      &\bfT_2 = -\frac{1}{\xi_2{n}}\bracket{\sum_{k = 2}^{L-1} \sum_{j = 1}^{k-1}{k \choose j}(k+1){c}_{k+1} c_1^j \eta_1^j \; \left(\bfa_0^{\odot (j+1)} \parant{\bfy_2 \odot (\bfX_2\hat\vbeta_1)^{\odot j }}^\top\right) \odot H_{k - j}\left(\widetilde\bfW_0\bfX_2^\top\right)}\bfX_2,\\[0.2cm]
    &\bfT_3 = -\frac{1}{\xi_2{n}}\bracket{\sum_{k = 1}^{L-1} (k+1){c}_{k+1} c_1^k \eta_1^k\; {\bfa_0^{\odot (k+1)}}\,\parant{\bfy_2 \odot (\bfX_2\hat\vbeta_1)^{\odot k}}^\top}\bfX_2.
\end{align*}
In what follows, we simplify the terms $\bfT_1$ and $\bfT_2$.

\paragraph{Term $\bfT_1$.} The matrix $\bfT_1$ can be written as
\begin{align*}
    \bfT_1 &= -\frac{1}{\xi_2{n}}\bracket{\parant{\bfa_0\bfy_2^\top} \odot {\sigma'\parant{{\bfW}_0\bfX_2^\top}}}\bfX_2\\ &=-{c}_1\,\bfa_0\parant{\frac{1}{\xi_2{n}}\bfX_2^\top \bfy_2}^\top - 
    \underbrace{\frac{1}{\xi_2{n}}\bracket{\parant{\bfa_0\bfy_2^\top} \odot {\sigma'_\perp\parant{{\bfW}_0\bfX_2^\top}}}\bfX_2}_{:=\bfE_2},
\end{align*}
where we expanded $\sigma'(z) = c_1 + \sigma_\perp(z)$ with $\Ex_{z\sim \normal(0,1)} \sigma_\perp(z) = 0$. For $\bfE_2$, we have
\begin{align*}
    \bfE_2 &= \frac{1}{\xi_2{n}}\bracket{\parant{\bfa_0\bfy_2^\top} \odot {\sigma'_\perp\parant{{\bfW}_0\bfX_2^\top}}}\bfX_2 =\frac{1}{\xi_2{n}} \bracket{\diag(\bfa_0)\;\sigma'_\perp\parant{{\bfW}_0\bfX_2^\top} \diag(\bfy)}\bfX_2,
\end{align*}
where we used Lemma~\ref{lemma:rank1-hadamard}. From Lemma~\ref{lemma:easy_norms}, we have
\begin{align*}
     &\|\diag(\bfa_0)\|_{\rm op} = \|\bfa_0\|_\infty =  \tilde{O}_{\mathbb{P}}\parant{\frac{1}{\sqrt{N}}}, \quad \|\bfX_2\|_{\rm op} = {O}_{\mathbb{P}}({\sqrt{N}}),\\[0.1cm]&\text{and} \quad \|\diag(\bfy_2)\|_{\rm op} = \|\bfy_2\|_{\infty} = \tilde{O}_{\mathbb{P}}(1).\\[-0.2cm]
\end{align*}
Moreover,  $\|\sigma_\perp'(\bfW_0\bfX_2^\top)\|_{\rm op} = O_{\mathbb{P}}(\sqrt{N})$, because using the property that  $\frac{d}{d x} H_k(x)=k H_{k-1}(x)$ \cite[Equation 22.8.8]{abramowitz1968handbook}, we can write
\begin{align*}
\|\sigma_\perp'(\bfW_0\bfX_2^\top)\|_{\rm op}  &= \left\|\sum_{k = 1}^{L-1}c_{k+1}(k+1)H_k(\bfW_0\bfX_2^\top)\right\|_{\rm op}\\[0.2cm]
&\leq \sum_{k = 1}^{L-1}|c_{k+1}|(k+1)\left\|H_k(\bfW_0\bfX_2^\top)\right\|_{\rm op} = \tilde\Theta_{\mathbb{P}}(\sqrt{N}),
\end{align*}
where we have used Lemma~\ref{lemma:hermite_norm}. Hence, using the sub-multiplicativity of the operator norm:
\begin{align*}
    \|\bfE_2\|_{\rm op} &\leq \frac{1}{\xi_2 n} \|\bfa_0\|_{\infty}  \|\sigma_\perp'(\bfW_0\bfX_2^\top)\|_{\rm op}\|\bfy\|_{\infty}\|\bfX_2\|_{\rm op}\\[0.1cm]
    &= \tilde{O}_{\mathbb{P}}\parant{ \frac{1}{\xi_2{n}} \times \frac{1}{\sqrt{N}} \times \sqrt{N} \times \sqrt{N}} = \tilde{O}_{\mathbb{P}}\parant{\frac{1}{\sqrt{N}}},
\end{align*}

 Thus, the first term  $\bfT_1$ is given by
\begin{align*}
    \bfT_1 =-{c}_1\,\bfa_0\parant{{\frac{1}{\xi_2{n}}\bfX_2^\top \bfy_2}}^\top - \bfE_2
\end{align*}
in which $\|\bfE_2\|_{\rm op} = \tilde{O}_{\mathbb{P}}\parant{\frac{1}{\sqrt{N}}}$.

\paragraph{Term $\bfT_2$.} We write $\bfT_2$ as $\bfT_{2}= -\sum_{k = 2}^{L-1}\sum_{j = 1}^{k-1}{k \choose j}(k+1){c}_{k+1}\bfT_{2;k,j}$ with $\bfT_{2;k,j}$  given by
\begin{align*}
    \bfT_{2;k,j} &= \bracket{\frac{ c_1^j \eta_1^j}{\xi_2{n}} \; \left({\bfa_0^{\odot j+1}} \parant{\bfy_2 \odot (\bfX_2\bfQ_1^{-1}\hat\vbeta_1)^{\odot j }}^\top\right) \odot H_{k - j}\left(\widetilde\bfW_0\bfX_2^\top\right)}\bfX_2
\end{align*}
Using the sub-multiplicativity of operator norm, we have
\begin{align*}
    \|\bfT_{2;k,j}\|_{\rm op} &\leq \frac{c_1^j\eta_1^j}{\xi_2{n}} \cdot \left\|{\bfa_0^{\odot (j+1)}} \right\|_{\infty} \cdot \left\|\bfy_2 \odot (\bfX_2\hat\vbeta_1)^{\odot j }\right\|_{\infty}\cdot \|H_{k-j}(\widetilde\bfW_0\bfX_2^\top)\|_{\rm op} \cdot \|\bfX_2\|_{\rm op}.
\end{align*}
It follows from Lemma~\ref{lemma:easy_norms} that
\begin{align*}
    \|\bfa_0^{\odot (j+1)}\|_\infty =  \tilde\Theta_{\mathbb{P}}\parant{{N^{-\frac{j+1}{2}}}}, \quad \|\bfX_2\|_{\rm op} = \Theta_{\mathbb{P}}\parant{\sqrt{N}}, \quad \text{and} \quad \Big\|\bfy_2 \odot (\bfX_2\hat\vbeta_1)^{\odot j }\Big\|_{\infty} = \tilde\Theta_{\mathbb{P}}(1).
\end{align*}
Also since $k\neq j$, we have $\|H_{k-j}(\widetilde\bfW_0\bfX_2^\top)\|_{\rm op} = \tilde\Theta_{\mathbb{P}}(\sqrt{N})$ by Lemma~\ref{lemma:hermite_norm}. Putting these together, we have
\begin{align*}
    \|\bfT_{2;k,j}\|_{\rm op}  = \tilde{O}_{\mathbb{P}}\parant{\frac{1}{\sqrt{N}}\parant{\frac{\eta_1}{\sqrt{N}}}^j}.
\end{align*}
Recall that $\eta_1 \asymp N^{\alpha_1}$ with $\alpha_1<1/2$. Thus, we have $\|\bfT_{2}\|_{\rm op} = o_{\mathbb{P}}\parant{1/\sqrt{N}}$.

\paragraph{Putting Everything Together.} Using~\eqref{eq:grad_new}, the gradient can be written as
\begin{align*}
    \nabla_\bfW \widehat{\mathcal{L}}_{\rm Corr}(\bfW_1, \bfa_0;\bfX_2, \bfy_2) &= -\frac{1}{\xi_2{n}}\bracket{\sum_{k = 1}^{L-1} (k+1){c}_{k+1} c_1^k \eta_1^k\; {\bfa_0^{\odot (k+1)}}\,\parant{\bfy_2 \odot (\bfX_2\hat\vbeta_1)^{\odot k}}^\top}\bfX_2\\ &\hspace{3cm}-{c}_1\,\bfa_0\parant{{\frac{1}{\xi_2{n}}\bfX_2^\top \bfy_2}}^\top + \bfE_3\\
    &\hspace{-3cm}= -\frac{1}{\xi_2{n}}\bracket{\sum_{k = 0}^{L-1} (k+1){c}_{k+1} c_1^k \eta_1^k\; {\bfa_0^{\odot (k+1)}}\,\parant{\bfy_2 \odot (\bfX_2\hat\vbeta_1)^{\odot k}}^\top\bfX_2}+ \bfE_3,
\end{align*}
where $\|\bfE_3\|_{\rm op} = \tilde{O}_{\mathbb{P}}\parant{\frac{1}{\sqrt{N}}}$, completing the proof.
\end{proof}

\section{Proof of Theorem~\ref{thm:W2_expansion}}
\begin{proof}
From Theorem~\ref{thm:second-gradient}, we have
\begin{align*}
    \nabla_\bfW \widehat{\mathcal{L}}_{\rm Corr}(\bfW_1, \bfa_0;\bfX_2, \bfy_2) = -\frac{1}{\xi_2{n}}\bracket{\sum_{k = 0}^{L-1} (k+1){c}_{k+1} c_1^k \eta_1^k\; {\bfa_0^{\odot (k+1)}}\,\parant{\bfy_2 \odot (\bfX_2\hat\vbeta_1)^{\odot k}}^\top\bfX_2}+ \bfE_3
\end{align*}
with $\|\bfE_3\|_{\rm op} = \tilde{O}_{\mathbb{P}}\parant{\frac{1}{\sqrt{N}}}$. Recall that $\eta_2 = N^{\alpha_2}$ with $\alpha_2 < 1/2$. Hence, we have
\begin{align*}
    \bfW_2 &= \bfW_1 + \eta_2\nabla_\bfW \widehat{\mathcal{L}}_{\rm Corr}(\bfW_1, \bfa_0;\bfX_2, \bfy_2)\\
    &= \bfW_1 + {\sum_{k = 0}^{L-1} \frac{\eta_2}{\xi_2{n}}(k+1){c}_{k+1} c_1^k \eta_1^k\; {\bfa_0^{\odot (k+1)}}\,\parant{\bfy_2 \odot (\bfX_2\hat\vbeta_1)^{\odot k}}^\top\bfX_2}+ \eta_2\bfE_3
\end{align*}
with $\|\eta_2\bfE_3\|_{\rm op} = o_{\mathbb{P}}(1)$. Denote the $k$-th term of the sum as $\bfO_k$. We have 
\begin{align*}
    \|\bfO_{k}\|_{\rm op} &=  (k+1) {c}_{k+1} c_1^k \eta_1^k\eta_2 \;\left\|\bfa_0^{\odot (k+1)} \right\|_2 \cdot \left\|{\frac{1}{\xi_2{n}}\bfX_2^\top\parant{\bfy_2 \odot (\bfX_2\hat\vbeta_1)^{\odot k}}}\right\|_2.
\end{align*}
From Lemma~\ref{lemma:easy_norms}, we can upper bound the first norm as
\begin{align*}
    \left\|\bfa_{0}^{\circ (k+1)}\right\|_{2} = \Theta_{\mathbb{P}}(N^{-k/2}).
\end{align*}
Moreover, we can write
\begin{align*}
    \left\|{\frac{1}{\xi_2{n}}\bfX_2^\top\parant{\bfy_2 \odot (\bfX_2\bfQ_1^{-1}\hat\vbeta_1)^{\odot k}}}\right\|_2 \leq \frac{1}{\xi_2 n} \|\bfX_2\|_{\rm op}\cdot\left\|{\bfy_2 \odot (\bfX_2\hat\vbeta_1)^{\odot k}}\right\|_2.
\end{align*}
Note that we readily have $\|\hat\vbeta_1\|_2 = \Theta_{\mathbb{P}}(1)$. Thus, $\bfX_2 \hat\vbeta_1$ is element-wise $\Theta_{\mathbb{P}}(1)$. The vector $\bfy_2$ is also element-wise $\Theta_{\mathbb{P}}(1)$. Hence, 
\begin{align*}
    \left\|{\bfy_2 \odot (\bfX_2\hat\vbeta_1)^{\odot k}}\right\|_2= \Theta_{\mathbb{P}}(\sqrt{\xi_2 n}) = \Theta_{\mathbb{P}}(\sqrt{N})
\end{align*}
Putting everything together, we proved that $\|\bfO_k\|_{\rm op} = \Theta_{\mathbb{P}}(\eta_1^k\eta_2 N^{-k/2})$. Recall form Section~\ref{sec:conditions} that $\eta_1 \asymp N^{\alpha_1}$ and $\eta_2 \asymp N^{\alpha_2}$. Thus
\begin{align*}
    \|\bfO_k\|_{\rm op} = \Theta_{\mathbb{P}}\parant{N^{(\alpha_1 - 1/2)k + \alpha_2}}
\end{align*}
which is $o(1)$ for all $k$ satisfying
\begin{align*}
    k > \Lambda(\alpha_1, \alpha_2) := \min\parant{L-1,  \left\lfloor\frac{\alpha_2}{\frac{1}{2} - \alpha_1}\right\rfloor}.
\end{align*} 
Hence, $\bfW_2$ can be approximated with $\bar\bfW_2$ given by
\begin{align*}
    \bar\bfW_2 =  \bfW_0 &+ c_1  \bfa_0 \parant{\frac{\eta_1}{\xi_1 n} (\bfX_1^\top \bfy_1) + \frac{\eta_2}{\xi_2 n} (\bfX_2^\top \bfy_2)}^\top \\[0.2cm]
    &+  \sum_{k = 1}^{\Lambda(\alpha_1, \alpha_2)}{(k+1){c}_{k+1} c_1^k \eta_2\eta_1^k\; \bfa_0^{\odot (k+1)}\,
    \bracket{\frac{1}{\xi_2{n}}\bfX_2^\top\parant{\bfy_2 \odot \parant{\frac{1}{\xi_1 n}\bfX_2\bfX_1^\top \bfy_1}^{\odot k}}}
    }^\top,
\end{align*}
with $\|\bfW_2 - \bar\bfW_2\|_{\rm op} = o_{\mathbb{P}}(1)$, proving the theorem.
\end{proof}

\section{Proof of Theorem~\ref{thm:reuse}}
\begin{proof}
In the reused batch setting, $\bfX_1 = \bfX_2 = \bfX$ and $\bfy_1  =\bfy_2 = \bfy$. Thus, $\hat\vbeta_{2;q}$ is given by
\begin{align*}
    \hat\vbeta_{2;q}=\frac{1}{ n}\bfX^\top\parant{\bfy \odot \parant{\frac{1}{n} \bfX\bfX^\top  \bfy}^{\odot q}}.
\end{align*}
First, note that the vector $\hat\vbeta_{2;k}$ can be written as follows

\begin{align*}
    \vbeta_{\star,p}^\top \hat\vbeta_{2;q} &= \frac{1}{ n}(\bfX\vbeta_{\star,p})^\top\parant{\bfy \odot \parant{\frac{1}{n} \bfX\bfX^\top  \bfy}^{\odot q}}\\[0.2cm]
    &=\Bigg[
        \bfx_1^\top \vbeta_{\star,p}   \Big|\cdots\Big|   \bfx_n^\top \vbeta_{\star,p} \Bigg]
        \begin{bmatrix}
            y_1 \parant{\frac{1}{n}\bfx_1^\top\sum_{j=1}^{n}\bfx_jy_j}^q\\
            \vdots\\
            y_n \parant{\frac{1}{n}\bfx_n^\top\sum_{j=1}^{n}\bfx_jy_j}^q
        \end{bmatrix} = \frac{1}{n} \sum_{i = 1}^{n} \bracket{y_i(\bfx_i^\top \vbeta_{\star,p})  \parant{\frac{1}{n}\sum_{j  =1}^{n} {\bfx_i^\top\bfx_j }\, y_j}^k}.
\end{align*}
Now, consider the inner sum. We have
\begin{align*}
    \frac{1}{n} \sum_{j  =1}^{n} {\bfx_i^\top\bfx_j }\, y_j &= \parant{\phi \, y_i + o(1)} + \bfx_i^\top\parant{\frac{1}{n}\sum_{j \neq i} {\bfx_j }\, y_j}.
\end{align*}
where we have used the fact that 
\begin{align*}
    \frac{1}{n}\bfx_i^\top \bfx_i= \frac{1}{n} \Ex\parant{\bfx_i^\top \bfx_i} + o(1) = \frac{1}{n} \trace{\bfI_{\dx}} + o(1) = \phi + o(1).
\end{align*}
Also, note that 
\begin{align*}
    \Ex\bracket{\frac{1}{n}\sum_{j \neq i} {\bfx_j }\, y_j} = \bar\vbeta, \quad \text{given by} \quad \bar\vbeta := \sum_{k = 1}^{M} c_{g_k, 1}\vbeta_{\star, k}
\end{align*}
Hence, we write
\begin{align*}
    \frac{1}{n} \sum_{j  =1}^{n} {\bfx_i^\top\bfx_j }\, y_j &= \parant{\phi \, y_i + o(1)} + \bfx_i^\top \bar\vbeta + \bfx_i^\top\boldsymbol{\rho}_{-i}, \quad\text{where}\quad\boldsymbol{\rho}_{-i} := \parant{\frac{1}{n}\sum_{j \neq i} {\bfx_j }\, y_j - \bar\vbeta}.
\end{align*}
Note that  $\boldsymbol{\rho}_{-i}$ is independent from $\bfx_i$. Let $z_p = \bfx_i^\top \vbeta_{\star,p}$. Because the vectors $\vbeta_{\star,p}$ are orthonormal, the variables $z_1, \dots, z_M$ are independent and drawn from $\normal(0,1)$. Also, for any $p \in [M]$:
\begin{align*}
    \vbeta_{\star, p}^\top\boldsymbol{\rho}_{-i} = \frac{1}{n}\sum_{j \neq i} {(\vbeta_{\star, p}^\top\bfx_j) }\, y_j - c_{g_{p},1} = o_{\mathbb{P}}(1).
\end{align*}
Thus, defining the random variable $G := \bfx_i^\top \boldsymbol{\rho}_{-i}$, the variable $G$ is also independent from $z_1, \dots, z_m$ and $G \sim \normal(0, \|\boldsymbol{\rho}_{-i}\|_2^2)$. The norm of this vector coverages to a deterministic limit, which is given in the following lemma.
\begin{lemma}
    \label{lemma:norm}
    In the setting of Theorem~\ref{thm:reuse}, we have
    \begin{align*}
        \|\boldsymbol{\rho}_{-i}\|_2^2 = \phi\,  \Ex [y^2].
    \end{align*}
    where $y = \ep + \sum_{k = 1}^{M}g_k(z_k) $ in which $z_1, \dots, z_M \stackrel{\mathrm{i.i.d.}}{\sim} \normal(0,1)$ and $\ep \sim \normal(0, \sigma_\ep^2)$  are independent random variables.
\end{lemma} 
We defer the proof of this lemma to Section~\ref{sec:proof_lemma}. Hence, putting everything together and using the weak law of large numbers:
\begin{align*}
    \vbeta_{\star,p}^\top \hat\vbeta_{2;q}  &= \frac{1}{n} \sum_{i = 1}^{n} \bracket{y_i(\bfx_i^\top \vbeta_{\star,p})  \parant{\frac{1}{n}\sum_{j  =1}^{n} {\bfx_i^\top\bfx_j }\, y_j}^q}\\
    &= \frac{1}{n} \sum_{i = 1}^{n} \bracket{y_i z_p  \parant{\phi y_i + G + \sum_{t = 1}^{M}c_{g_t, 1}z_t}^q} \to \Ex\bracket{yz_p\parant{\phi y + G + \sum_{t = 1}^{M}c_{g_t, 1}z_t}^q}
\end{align*}
where $z_1, \dots, z_M \stackrel{\mathrm{i.i.d.}}{\sim} \normal(0,1)$, $\ep \sim \normal(0, \sigma_\ep^2)$, and $G \sim \normal(0, \phi\,\Ex[y^2])$ are independent random variables, and $y = \ep + \sum_{k = 1}^{M}g_k(z_k) $. This completes the proof of the theorem.
\end{proof}
\section{Proof of Lemma~\ref{lemma:norm}}
\label{sec:proof_lemma}
\begin{proof}
The squared norm of $\boldsymbol{\rho}_{-i}$ is given by
\begin{align}
    \label{eq:norm}
    \|\boldsymbol{\rho}_{-i}\|_2^2 = \|\bar\vbeta\|_2^2 - 2 \bar\vbeta^\top \parant{\frac{1}{n}\sum_{j \neq i} {\bfx_j }\, y_j} + \left\|\frac{1}{n}\sum_{j \neq i} {\bfx_j }\, y_j\right\|_2^2 
\end{align}
In what follows, we compute the limit of each these terms.
\paragraph{First Term.}  For this term in \eqref{eq:norm}, we recall that the vectors $\{\vbeta_{\star, k}\}_{k \in [M]}$ are orthonormal and $\bar\vbeta = \sum_{k = 1}^{M}c_{g_k, 1}\vbeta_{\star, k}$. Hence, we readily have
\begin{align*}
    \|\bar\vbeta\|_2^2 = \sum_{k = 1}^{M}c_{g_k, 1}^2.
\end{align*}

\paragraph{Second Term.} The second term in \eqref{eq:norm} can be written as
\begin{align*}
    - 2 \bar\vbeta^\top \parant{\frac{1}{n}\sum_{j \neq i} {\bfx_j }\, y_j}  = -\frac{2}{n} \sum_{j \neq i}\sum_{k = 1}^{M} c_{g_k, 1} ({\bfx_j }^\top \vbeta_{\star, k})\, y_j.
\end{align*}
Note that $z_k := \bfx_j^\top \vbeta_{\star, k} \sim \normal(0,1)$. Thus, using the weak law of large numbers, with probability $1-o(1)$ we have
\begin{align*}
    - 2 \bar\vbeta^\top \parant{\frac{1}{n}\sum_{j \neq i} {\bfx_j }\, y_j} \to -2\sum_{k = 1}^{M}  c_{g_k, 1}\Ex\bracket{z_k y}
\end{align*}
where $y = \ep + \sum_{t=1}^{M}g_t(z_t)$ in which $z_1, \dots, z_M\sim \normal(0,1)$ and $\ep \sim \normal(0, \sigma_\ep^2)$ are independent random variables, because the vectors $\vbeta_{\star, 1}, \dots, \vbeta_{\star, M}$ are orthonormal.  We can further simplify the above expectation 
\begin{align*}
    \Ex [z_k y] = \Ex\bracket{z_k\ep + \sum_{s=1}^{M}z_kg_s(z_s) } = c_{g_k, 1}.
\end{align*}
Hence, putting everything together, we have
\begin{align*}
     - 2 \bar\vbeta^\top \parant{\frac{1}{n}\sum_{j \neq i} {\bfx_j }\, y_j} = -2 \sum_{k = 1}^{M} c_{g_k, 1}^2 + o_{\mathbb{P}}(1).
\end{align*}

\paragraph{Third Term.}  The third term in \eqref{eq:norm} is given by
\begin{align*}
    \left\|\frac{1}{n}\sum_{j \neq i} {\bfx_j }\, y_j\right\|_2^2 = \left\|\frac{1}{n}\sum_{j = i}^{n} {\bfx_j }\, y_j\right\|_2^2 + o_{\mathbb{P}}(1) = n^{-2}\parant{\bfy^\top \bfX\bfX^\top \bfy} + o_{\mathbb{P}}(1).
\end{align*}
Recall that $\bfy = \boldsymbol{\ep} + \sum_{k = 1}^{M}g_k(\bfX\vbeta_{\star, k})$. Hence, recalling the independence of $\boldsymbol{\ep}$ from all other randomness in the problem, using the weak law of large number we have
\begin{align*}
    n^{-2}\parant{\bfy^\top \bfX\bfX^\top \bfy} = n^{-2}\parant{\boldsymbol{\ep}^\top \bfX\bfX^\top \boldsymbol{\ep}} + n^{-2}\sum_{p = 1}^{M}\sum_{q = 1}^{M}\parant{g_p(\bfX\vbeta_{\star, p})}^\top \bfX\bfX^\top\parant{g_q(\bfX\vbeta_{\star, q})} + o_{\mathbb{P}}(1).
\end{align*}
For the first term consisting of the noise $\boldsymbol{\ep}$, we can use the Hanson-Wright inequality \cite{hanson1971bound} to write
\begin{align*}
    n^{-2} \boldsymbol{\ep}^\top\bfX \bfX^\top\boldsymbol{\ep} &=   \sigma_\ep^2 n^{-1} \trace(\bfX\bfX^\top/n) + o_{\mathbb{P}}(1)\\[0.2cm]
    &= \sigma_\ep^2 n^{-1}  \trace(\bfI_{\dx}) + o_{\mathbb{P}}(1)= \sigma_\ep^2  \phi + o_{\mathbb{P}}(1).
\end{align*}
For the second term consisting of signals, we split it into two cases with $p = q$ and $p \neq q$ and analyze the seperately.
\paragraph{Case of $p = q$: } We can decompose the function $g_p(z)$ as $g_p(z) = c_{g_p, 1}z + g_{p, \perp}(z)$ and write
 \begin{align}
        \label{eq:sum1}
        n^{-2} g_p(\bfX\vbeta_{\star, p})^\top\bfX \bfX^\top g_p(\bfX\vbeta_{\star, p}) &=  c_{g_p,1}^2 n^{-2} \vbeta_{\star, p}^\top (\bfX^\top \bfX)^2 \vbeta_{\star, p}\nonumber\\ &\hspace{1cm}+ 2  c_{g_p,1} n^{-2} (\bfX\vbeta_{\star, p})^\top \bfX\bfX^\top g_{p, \perp}(\bfX\vbeta_{\star, p})\nonumber\\ &\hspace{1cm}+  n^{-2} g_{p, \perp}(\bfX\vbeta_{\star, p})^\top \bfX\bfX^\top g_{p, \perp}(\bfX\vbeta_{\star, p}).
    \end{align}
For the first term in this sum, by the Hanson-Wright inequality, we can write 
    \begin{align*}
        c_{g_p,1}^2 n^{-2} \vbeta_{\star, p}^\top (\bfX^\top \bfX)^2 \vbeta_{\star, p} &= c_{g_p,1}^2\,\dx^{-1} \trace(n^{-2}(\bfX^\top\bfX)^2) + o_{\mathbb{P}}(1) = c_{g_p,1}^2 (1+\phi) + o_{\mathbb{P}}(1)
    \end{align*}
    where for last equality we plugged in the second Wishart moment. 
  
    For the second term in  \eqref{eq:sum1}, although by construction $\bfX\vbeta_{\star, p}$ and $g_{p, \perp}(\bfX\vbeta_{\star, p})$ are uncorrelated, the vector $\bfX\vbeta_{\star, p}$ and the matrix $\bfX\bfX^\top$ are in fact dependent. This further complicates the analysis. To resolve this issue, we define $\tilde \bfX = \bfX - \bfX \vbeta_{\star, p}\vbeta_{\star, p}^\top$ which satisfies $\bfX\vbeta_{\star, p} \independent \tilde\bfX$ and write
    \begin{align*}
        \tilde\bfX\tilde\bfX^\top = \bfX\bfX^\top - (\bfX\vbeta_{\star,p})(\bfX\vbeta_{\star,p})^\top.
    \end{align*}
    Hence, we can rewrite
    \begin{align*}
         2  c_{g_p,1} n^{-2} (\bfX\vbeta_{\star, p})^\top \bfX\bfX^\top g_{p, \perp}(\bfX\vbeta_{\star, p}) &= 2  c_{g_p,1} n^{-2} (\bfX\vbeta_{\star, p})^\top  {\tilde\bfX\tilde\bfX^\top } g_{p, \perp}(\bfX\vbeta_{\star, p}) \\[0.1cm]
         &\hspace{1cm} + 2  c_{g_p,1} n^{-2} (\bfX\vbeta_{\star, p})^\top  {\bfX\vbeta_{\star, p}\vbeta_{\star, p}^\top\bfX^\top}g_{p, \perp}(\bfX\vbeta_{\star, p})
    \end{align*}
    The first term can be shown to be $o_{\mathbb{P}}(1)$ using the weak law of large numbers, by  noting that $\tilde\bfX\tilde\bfX^\top$ is independent of $\bfX\vbeta_{\star,p}$ and using the fact that  $\bfX\vbeta_{\star,p}$ and $g_{p, \perp}(\bfX\vbeta_{\star,p})$ are orthogonal by construction. 
    For the second term, note that by construction of the $g_{p, \perp}$, and using the weak law of large numbers, $n^{-1}\vbeta_{\star, p}^\top \bfX^\top g_{p, \perp}(\bfX\vbeta_{\star,p}) = o_{\mathbb{P}}(1)$, whereas $n^{-1}(\bfX\vbeta_{\star, p})^\top \bfX\vbeta_{\star, p} = \Theta_{\mathbb{P}}(1)$. Hence, with probability $1 - o(1)$, we have
    \begin{align*}
       2  c_{g_p,1} n^{-2} (\bfX\vbeta_{\star, p})^\top \bfX\bfX^\top g_{p, \perp}(\bfX\vbeta_{\star, p}) = o(1).
    \end{align*}
    For the third term in \eqref{eq:sum1}, we can use a similar argument and replace $\bfX$ with $\tilde\bfX$ to show that
    \begin{align*}
        n^{-2} g_{p, \perp}(\bfX\vbeta_\star)^\top \bfX\bfX^\top g_{p,\perp}(\bfX\vbeta_\star) \to \Ex_{z\sim \normal(0,1)} [g_{p, \perp}(z)]^2 \, n^{-1} \trace(\bfI_{\dx}) = \phi \Ex_{z\sim \normal(0,1)} [g_{p, \perp}(z)]^2.
    \end{align*}
    Putting everything together yields
    \begin{align*}
        n^{-2} g_p(\bfX\vbeta_{\star, p})^\top\bfX \bfX^\top g_p(\bfX\vbeta_{\star, p}) = c_{g_p,1}^2 (1+\phi) + \phi \Ex_{z\sim \normal(0,1)} [g_{p, \perp}(z)]^2 + o_{\mathbb{P}}(1).
    \end{align*}
\paragraph{Case of $p \neq q$: } Now let $p \neq q$ and consider
$\parant{g_p(\bfX\vbeta_{\star, p})}^\top \bfX\bfX^\top\parant{g_q(\bfX\vbeta_{\star, q})}.$ Similar to the $p = q$ case, we write
\begin{align}
        \label{eq:sum1w}
        n^{-2} g_p(\bfX\vbeta_{\star, p})^\top\bfX \bfX^\top g_q(\bfX\vbeta_{\star, q}) &=  c_{g_p,1}c_{g_q, 1} n^{-2} \vbeta_{\star, p}^\top (\bfX^\top \bfX)^2 \vbeta_{\star, q}\nonumber\\ &\hspace{1cm}+ 2  c_{g_p,1} n^{-2} (\bfX\vbeta_{\star, p})^\top \bfX\bfX^\top g_{q, \perp}(\bfX\vbeta_{\star, q})\nonumber\\ &\hspace{1cm}+  n^{-2} g_{p, \perp}(\bfX\vbeta_{\star, p})^\top \bfX\bfX^\top g_{q, \perp}(\bfX\vbeta_{\star, q}).
\end{align}
The first term is $o_{\mathbb{P}}(1)$ because $p \neq q$ and $\vbeta_{\star, p}$ and $\vbeta_{\star, q}$ are orthogonal. For the second and third term, we similar to the case of $p = q$, we replace $\bfX$ with $\tilde\bfX$ and show that they are also both $o_{\mathbb{P}}(1)$.

\paragraph{Putting Everything Together,}we have
\begin{align*}
    \sum_{p ,q}\parant{g_p(\bfX\vbeta_{\star, p})}^\top \bfX\bfX^\top\parant{g_q(\bfX\vbeta_{\star, q})} &= \sum_{p = 1}^{M}\bracket{c_{g_p,1}^2 (1+\phi) + \phi \Ex_{z\sim \normal(0,1)} [g_{p, \perp}(z)]^2}\\ &= \sum_{p = 1}^{M}\bracket{c_{g_p,1}^2 + \phi \Ex_{z\sim \normal(0,1)} [g_{p}^2(z)]}
\end{align*}
Thus, the norm of $\boldsymbol{\rho}_{-i}$ is given by
\begin{align*}
    \left\|\boldsymbol{\rho}_{-i}\right\|_2^2 = \sigma_\ep^2 \phi  + \sum_{p = 1}^{M}\bracket{c_{g_p,1}^2  + \phi \Ex_{z\sim \normal(0,1)} [g_{p}^2(z)]} + o_{\mathbb{P}}(1) = \phi\,\Ex[y^2] + o_{\mathbb{P}}(1),
\end{align*}
where $y = \ep + \sum_{k = 1}^{M}g_k(z_k) $ in which $z_1, \dots, z_M \stackrel{\mathrm{i.i.d.}}{\sim} \normal(0,1)$ and $\ep \sim \normal(0, \sigma_\ep^2)$  are independent random variables. This completes the proof of the lemma.
\end{proof}

\section{Proof of Theorem~\ref{thm:fresh}}
\begin{proof}
For simplicity, let $n_1 = \xi_1 n$ and $n_2 = \xi_2 n$. In the fresh batch setting, $\hat\vbeta_{2;q}$ is given by
\begin{align*}
    \hat\vbeta_{2;q}=\frac{1}{n_2}\bfX_2^\top\parant{\bfy_2 \odot \parant{\frac{1}{n_1} \bfX_2\bfX_1^\top  \bfy_1}^{\odot q}}.
\end{align*}
For simplicity of expressions, let $\{(\bfx_i, \bfy_i)\}_{i = 1}^{n_1}$ and $\{(\bar\bfx_i, \bar\bfy_i)\}_{i = 1}^{n_2}$ be the samples corresponding to $(\bfX_1, \bfy_1)$ and  $(\bfX_2, \bfy_2)$ respectively. Similar to the proof of Theorem~\ref{thm:reuse}, we have
\begin{align*}
    \vbeta_{\star,p}^\top \, \hat\vbeta_{2;q} = \frac{1}{n_1} \sum_{i = 1}^{n_2} \bracket{\bar{y}_i(\bar\bfx_i^\top\vbeta_{\star,p})  \bracket{\,\bar\bfx_i^\top\parant{ \frac{1}{n_1}\sum_{j  =1}^{n_1} {\bfx_j }\, y_j}}^q\,~}.
\end{align*}
Note that
\begin{align*}
    \Ex\bracket{\frac{1}{n}\sum_{j = 1}^{n_1} {\bfx_j }\, y_j} = \bar\vbeta, \quad \text{given by} \quad \bar\vbeta := \sum_{k = 1}^{M} c_{g_k, 1}\vbeta_{\star, k}
\end{align*}
Hence, we write
\begin{align*}
    \frac{1}{n} \sum_{j  =1}^{n} {\bar\bfx_i^\top\bfx_j }\, y_j &= \bar\bfx_i^\top \bar\vbeta + \bar\bfx_i^\top\boldsymbol{\rho}, \quad\text{where}\quad\boldsymbol{\rho} := \parant{\frac{1}{n_1}\sum_{j = 1}^{n} {\bfx_j }\, y_j - \bar\vbeta}.
\end{align*}
Note that  $\boldsymbol{\rho}$ is independent from $\bar\bfx_i$ because an independent batch is used for the first step. Let $z_p = \bar\bfx_i^\top \vbeta_{\star,p}$. Because the vectors $\vbeta_{\star,p}$ are orthonormal, the variables $z_1, \dots, z_M$ are independent and drawn from $\normal(0,1)$. Also, for any $p \in [M]$:
\begin{align*}
    \vbeta_{\star, p}^\top\boldsymbol{\rho}= \frac{1}{n}\sum_{j \neq i} {(\vbeta_{\star, p}^\top\bfx_j) }\, y_j - c_{g_{p},1} = o_{\mathbb{P}}(1).
\end{align*}
Thus, defining the random variable $G := \bar\bfx_i^\top \boldsymbol{\rho}$, the variable $G$ is also independent from $z_1, \dots, z_m$ and $G \sim \normal(0, \|\boldsymbol{\rho}\|_2^2)$. The norm of this vector coverages to a deterministic limit, which is given in Lemma~\ref{eq:norm}. Note that $\|\boldsymbol{\rho}\|_2 = \|\boldsymbol{\rho}_{-i}\|_2 + o_{\mathbb{P}}(1)$. Hence, putting everything together and using the weak law of large numbers:
\begin{align*}
    \vbeta_{\star,p}^\top \hat\vbeta_{2;q}   \to \Ex\bracket{yz_p\parant{G + \sum_{t = 1}^{M}c_{g_t, 1}z_t}^q}
\end{align*}
where $z_1, \dots, z_M \stackrel{\mathrm{i.i.d.}}{\sim} \normal(0,1)$, $\ep \sim \normal(0, \sigma_\ep^2)$, $y = \ep + \sum_{k = 1}^{M}g_k(z_k) $. Also $G \sim \normal(0, \phi\,\Ex[y^2]/\xi_2)$ is an independent random variable. This completes the proof of the theorem.
\end{proof}

\end{document}